%% file: Canconical_Hamiltonian_Lifting_Stable.tex
\providecommand{\figurepath}{./Figures}

\documentclass[%
 onecolumn
   , hidempi
]{mpi2015-cscpreprint}

\usepackage[american]{babel}

\usepackage{graphicx}
\usepackage{wrapfig}
\usepackage{amssymb}
\usepackage{amsthm}
\usepackage[mathscr]{eucal}
\usepackage{amsmath}
\usepackage{breqn}

\numberwithin{equation}{section}
\numberwithin{figure}{section}
\numberwithin{table}{section}

\usepackage{import}
\usepackage{xifthen}
\usepackage{pdfpages}
\usepackage{placeins}
\usepackage{transparent}
\usepackage{cleveref}
\usepackage{sidecap}    
\usepackage{floatrow}
\usepackage{float}
\usepackage{multicol}
\usepackage{todonotes}
\usepackage{algorithm}
\usepackage[noend]{algpseudocode}
\renewcommand{\arraystretch}{1.2}

\makeatletter
\def\algbackskip{\hskip-\ALG@thistlm}
\makeatother

\usepackage{caption}
\usepackage{subcaption}
\usepackage{arydshln}

\definecolor{lightgray}{gray}{0.9}
\usepackage{color}
\definecolor{bluegreen}{rgb}{0.0, 0.87, 0.87}

\newtheorem{lemma}{Lemma}
\newtheorem{remark}{Remark}
\newtheorem{definition}{Definition}
\newtheorem{example}{Example}

\newtheorem{theorem}{Theorem}

\newcommand*{\bl}[1]{\mathbf{#1}}

\newcommand{\Hamiltonian}{\mathcal{H}}
\DeclareMathOperator{\Diff}{d}

\newcommand{\linearembs}{{\texttt{s-linear-embs}}}
\newcommand{\quadembs}{{\texttt{quad-embs}}}
\newcommand{\cubicembs}{{\texttt{s-cubic-embs}}}
\newcommand{\opinf}{{\texttt{OpInf-Ham}}}

\newcommand{\ddt}{\dfrac{\mathrm d}{\mathrm{d} t}}
\usepackage{mymacros}
\usepackage{todonotes}
\usepackage{enumitem}

\begin{document}
  

\title{Deep Learning for Structure-Preserving Universal Stable Koopman-Inspired Embeddings for Nonlinear Canonical Hamiltonian Dynamics}

\author[$\ast\ast$]{Pawan Goyal}
\affil[$\ast\ast$]{Max Planck Institute for Dynamics of Complex Technical Systems, 39106 Magdeburg, Germany.\authorcr
	\email{goyalp@mpi-magdeburg.mpg.de}, \orcid{0000-0003-3072-7780}
}

\author[$\ast$]{Süleyman Y\i ld\i z}
\affil[$\ast$]{Max Planck Institute for Dynamics of Complex Technical Systems, 39106 Magdeburg, Germany.\authorcr
	\email{yildiz@mpi-magdeburg.mpg.de}, \orcid{0000-0001-7904-605X}
}
 
\author[$\dagger\ddagger$]{Peter Benner}
\affil[$\dagger\ddagger$]{Max Planck Institute for Dynamics of Complex Technical Systems, 39106 Magdeburg, Germany.\authorcr
  \email{benner@mpi-magdeburg.mpg.de}, \orcid{0000-0003-3362-4103}
}
\affil[$\dagger\ddagger$]{Otto von Guericke University,  Universit\"atsplatz 2, 39106 Magdeburg, Germany\authorcr
  \email{peter.benner@ovgu.de} 
  \vspace{-0.5cm}
}
  
\shorttitle{Structure-Preserving Universal Stable Embeddings}
\shortauthor{P. Goyal, S. Y\i ld\i z, P. Benner}
\shortdate{}
  

  
\abstract{%
	
		Discovering a suitable coordinate transformation for nonlinear systems enables the construction of simpler models, facilitating prediction, control, and optimization for complex nonlinear systems. To that end,  Koopman operator theory offers a framework for global linearization for nonlinear systems, thereby allowing the usage of linear tools for design studies. 	In this work, we focus on the identification of global linearized embeddings for canonical nonlinear Hamiltonian systems through a symplectic transformation. While this task is often challenging, we leverage the power of deep learning to discover the desired embeddings.	Furthermore, to overcome the shortcomings of Koopman operators for systems with continuous spectra, we apply the lifting principle and learn global cubicized embeddings. Additionally, a key emphasis is paid to enforce the bounded stability for the dynamics of the discovered embeddings. We demonstrate the capabilities of deep learning in acquiring compact symplectic coordinate transformation and the corresponding simple dynamical models, fostering data-driven learning of nonlinear canonical Hamiltonian systems, even those with continuous spectra.
}

\novelty{
	\begin{itemize}
		\item Inspired by the Koopman operator theory, in this work, we address a data-driven discovery of global linearized embeddings for canonical Hamiltonian systems via a symplectic transformation. 
		\item To address the limitation of Koopman operator theory for the systems with continuous spectra, we propose the discovery of global cubicized embeddings, influenced by the lifting principle. 
		\item An essential aspect of our study involves enforcing the dynamic stability of the learned embeddings.
		\item We harness the power of deep learning, particularly auto-encoders, to effectively learn the desired embeddings.
		\item To demonstrate the effectiveness of our proposed methodologies, we provide several examples, including those in high-dimension, to show their applications and capabilities.
	\end{itemize}
} 
\maketitle

\section{Introduction}

Hamiltonian mechanics provides a potent mathematical framework for describing the intricate behaviors of diverse physics phenomena across a spectrum of scientific disciplines. This framework finds its applications in physics, engineering, and chemistry. Noteworthy applications span from the realms of solid-state physics, mechanical structures, and robotics to celestial mechanics, climate modeling, and beyond. 
Hamiltonian systems are characterized by their ability to conserve fundamental physical quantities, such as energy, momentum, symplecticity, and symmetries. As a result, it makes them a natural framework for representing  a diverse range of physical phenomena \cite{leimkuhler2005simulating, marsden2013introduction,arnol2013mathematical}. Leveraging these inherent properties during the course of engineering studies (e.g., control, prediction, and optimization) and modeling enables us to attain our objectives effectively. 

To capture complex nonlinear dynamics, nonlinear modeling is a key tool. While physics-based nonlinear dynamic modeling has historically been employed to comprehend complex processes, it is a formidable challenge to describe the underlying dynamics completely and accurately. 
With the surge in sensor technology, data availability has been more than ever. This, in turn, has intensified the pursuit of developing data-driven dynamic models that capture the  underlying nonlinear behavior.

Learning nonlinear models using data has remarkably progressed in recent years---from classical techniques to neural networks, see, e.g., \cite{NarP90,rico1994continuous,VanM96,SuyVdM96,lennart1999system,brunton2016discovering,peherstorfer2016data,chen2018neural,raissi2019physics}. However, the quest for a comprehensive framework for nonlinear system modeling remains ongoing, particularly, towards utilizing physics-knowledge at an abstract level to enhance interpretability and generalizability. In this work, we seek to explore the characteristics of Hamiltonian dynamics. 
To that end, we begin by highlighting the Hamiltonian neural network (HNN) framework~\cite{greydanus2019hamiltonian}, in which the central idea is energy-based modeling. That is, instead of modeling the vector field with neural networks, the underlying Hamiltonian is learned using a neural network, and then the dynamical model is constructed so that it is symplectic. HNN has inspired several studies to generalize the framework to dissipative systems \cite{zhong2020dissipative}, control systems \cite{zhong2019symplectic}, port-Hamiltonian systems \cite{desai2021port,duong2021hamiltonian}, and Poisson systems \cite{jin2022learning}. Moreover, the operator inference framework \cite{peherstorfer2016data}  has been extended to Hamiltonian systems \cite{sharma2022hamiltonian,gruber2023canonical,Sharmaetal23,SK_LagrangianOPINF} but they require a prior model hypothesis. 
However, engineering design studies using nonlinear models are challenging. Therefore, encapsulating nonlinear dynamics within a weakly nonlinear framework, or even within a linear framework, is appealing since it offers us to employ tools for weakly nonlinear systems.

In this direction, Koopman operator theory \cite{koopman1931hamiltonian,Koopman32} provides a methodology, indicating that nonlinear systems can be globally linearized in suitable coordinate systems or embeddings, and it can be of an infinite dimensional. However, dealing with infinite-dimensional space, though conceptually intriguing, poses challenges, thus motivating us to obtain a finite-dimensional approximation. 
One popular technique is the dynamic mode decomposition (DMD) \cite{rowley2009spectral,schmid2010dynamic,kutz2016dynamic}, which provides a means to achieve this goal. Despite its popularity, it may fall short in capturing highly nonlinear dynamics as it relies on linear measurements. To further enrich measurements, Extended DMD (eDMD) \cite{williams2015data} has been developed by defining hand-design nonlinear observers and using them to learn a linear operator through the classical DMD approach. 
A parallel idea can be found in the variational approach to conformation dynamics \cite{noe2013variational,nuske2014variational,nuske2016variational}, which also enriches models with nonlinear observers, akin to the spirit of eDMD. Nevertheless, there may still be closure-related issues with regard to the Koopman invariant subspace \cite{brunton2016koopman}. Furthermore, these observers can be high-dimensional, which can be tackled using kernel methods~\cite{Williams_kernel_15}. However, this might lead to an un-interpretable outcome or present challenges in reconstructing the relevant quantity of the interest. 
Moreover, these above-mentioned methodologies do not explicitly make use of the properties related to canonical Hamiltonian dynamics present in the data; hence, the globally linearized models are prone to not adhering to the physical properties of Hamiltonian dynamics.

In this work, we focus on identifying global linearized Koopman embeddings for canonical Hamiltonian dynamics. For this, we harness the capabilities of deep neural networks (DNNs) as a powerful computational tool. In the past, DNNs have been used to learn Koopman embeddings for nonlinear systems, see, e.g., \cite{wehmeyer2018time,mardt2018vampnets,Takeishietal_17,yeung2019learning,otto2019linearly,li2017extended,lusch2018deep} but without making use of Hamiltonian properties.
Recently, the work \cite{zhang2022hamiltonian} discusses learning Koopman embeddings while focusing on energy preservation. However, it does not use the symplectic property of canonical Hamiltonian systems, and thus, it does not guarantee that the auto-encoder defining the embeddings is a symplectic auto-encoder, e.g., \cite{buchfink2021symplectic,yildizetal23}. Furthermore, it does not guarantee bounded stability for the dynamics of the embeddings. 
To address this, we propose identifying Koopman-inspired global linearized embedding for canonical Hamiltonian systems so that the embedding dynamics are also canonical, and a symplectic nonlinear transformation yields embeddings. Additionally, we discuss how to enforce the bounded stability for the embeddings, which is arguably the most critical property related to dynamical systems.

Furthermore, for nonlinear systems, exhibiting continuous spectra of eigenvalues, obtaining finite-dimensional  Koopman embeddings is rather difficult. To deal with it, the work \cite{lusch2018deep} proposed the use of an auxiliary neural network designed to model the eigenvalues of linearized operators. In this case, the dynamics for the embeddings are not purely linear, as the operator depends on the output of the neural network, which is actually nonlinear. 
To alleviate the dependency on the neural network, we alternatively deal with the systems of continuous spectra by constructing weakly nonlinear systems. It relies on the framework of the lifting principle. It guides us that smooth nonlinear systems, including continuous spectra ones, can be exactly written as polynomial systems with a finite-dimensional embedding. Moreover, polynomial systems can exactly be written as finite-dimensional quadratic or cubic polynomial systems, see, e.g., \cite{savageau1987recasting,morGu11}. 
These principles to learn quadratic systems or embeddings have been used to learn nonlinear dynamics systems, see \cite{qian2020lift}. Further, learning a suitable quadratic embedding using auto-encoders has been discussed in \cite{goyal2022generalized}, which has been extended to canonical Hamiltonian systems while carefully treating the properties of the Hamiltonian dynamics \cite{yildizetal23}. However, a major drawback of learning quadratic embeddings in the context of canonical Hamiltonian systems is the preservation of bounded stability. To that end, we discuss a global cubicized framework for canonical systems that not only allows the construction of stable dynamics for embeddings but aims to mitigate the problem with continuous spectra in the classical Koopman framework as well. By embracing this framework, we seek to enhance the understanding and applicability of canonical Hamiltonian systems within a broader context.

The remaining paper is organized as follows. In \Cref{sec:sym_transformation}, we recall the definition of canonical Hamiltonian systems and the property of a symplectic transformation that is used to enforce the auto-encoder to be symplectic. In \Cref{sec:data_driven_modeling}, we discuss a Koopman-inspired global linearized framework for canonical Hamiltonian systems while preserving their properties in the Koopman embeddings. We also emphasize how to ensure bounded stability for the dynamics of the embeddings. Moreover, we address the continuous spectra by means of learning weakly nonlinear systems---that is, a cubic nonlinear system with the Hamiltonian being a quartic function and discuss their stability as well. \Cref{sec:num} presents numerical experiments, illustrating the performance of the proposed methodologies and presenting a comparison with \cite{yildizetal23}. Additionally, in \Cref{sec:high_dimensional_case}, we discuss how to efficiently employ the proposed methodologies to high-dimensional data and present a comparison with \cite{Sharmaetal23,SK_LagrangianOPINF}. In \Cref{sec:conclusions}, we conclude the paper with a summary and avenues for future research. 
%
\section{Symplectic Transformation}\label{sec:sym_transformation}
Our interest lies in learning canonical Hamiltonian systems of the form:
\begin{equation}
\label{eq:HamiltonianEquations}
\dot{\bx}(t) = \bJ_{2n} \nabla_\bx \bH(\bx(t)) \in \R^{2n}, \quad \bx(0) = \bx_0,
\end{equation}
where the state $\bx \in \R^{2n}$ contains generalized position $\bq\in \Rn$ and generalized momenta $\bp\in \Rn$, and $\bJ_{2n} $ is a symplectic matrix, i.e.,
$\bJ_{2n} := \begin{bmatrix} \mathbf{0} & \bI_n\\ -\bI_n & \mathbf{0} \end{bmatrix} \in \R^{2n \times 2n}$,
and $\nabla_\bx$ denotes the gradient with respect to $\bx$. Moreover, $\bx_0 = (\bq_0,\bp_0) \in \R^{2n}$ denotes an initial condition. Furthermore, the Hamiltonian function $\bH \colon \R^{2n} \to \R$ defines the Hamiltonian of the system and is constant along the solution trajectories. 
Next, we recall a definition of a symplectic transformation from, e.g.,  \cite{yildizetal23}.
\begin{definition}[e.g., \cite{yildizetal23}]
	A map $\boldsymbol{\psi}$ is a symplectic transformation of $2n$-dimensional to $2m$-dimensional when the following condition is  fulfilled:
	\begin{equation}\label{eq:symplecticEmbedding}
	(\Diff \boldsymbol{\psi}_\bx)^\top \bJ_{2m} \Diff \boldsymbol{\psi}_\bx = \bJ_{2n}, \qquad \forall~\bx \in \R^{2n},
	\end{equation}
where 	$\Diff \boldsymbol{\psi}_\bx \in \R^{2m \times 2n}$ is the Jacobian of $ \boldsymbol{\psi}$ with respect to $\bx$. Furthermore, when $m \geq 0$, we  refer to it as a \emph{symplectic lifting} \cite{yildizetal23}. 
\end{definition}

\section{Data-Driven Canonical Hamiltonian Systems}\label{sec:data_driven_modeling}
Here, we set up our problem of learning canonical Hamiltonian systems using data. Specifically, in this work, we focus on continuous-time dynamical systems of the form:
\begin{equation}
\ddt \bx(t) = \mathbf f(\bx(t)),
\end{equation}
where $\bx(t)$ is the state of the system, and the function $\mathbf f$---also referred to as vector field---defines its time evolution for a given initial condition and is often nonlinear. Additionally, we assume the state $\bx(t)$ to be low-dimensional; however, for high-dimensional $\bx$, we can obtain a low-dimensional representation by leveraging coherent structures in the spatial space, which we shall discuss later in detail. Our primary objective is to learn the function $\mathbf{f}$ from the given data $\mathbf{x}$, where $\mathbf{f}$ can be arbitrarily complex and nonlinear.

In this work, we draw inspiration from the Koopman operator theory \cite{koopman1931hamiltonian} and lifting principles \cite{savageau1987recasting,morGu11} to achieve our goal. In particular, we aim to learn the time evolution of $\mathbf{x}$ through suitable observers, employing concepts from these theories to address our learning task.

\subsection{Koopman-inspired linear symplectic representations and its stability}\label{subsec:linear_stability}
The Koopman theory offers a valuable approach to represent nonlinear dynamics as linear dynamical systems by means of the Koopman operators, thus allowing predictions and control for nonlinear systems by employing tools from linear theory. In particular, the Koopman theory indicates that it is possible to define a set of observers $\by$ as a function of $\bx$, such that the evolution of these observers can be described by a linear dynamical system, i.e., 
\begin{equation}
	\dfrac{d}{dt} \by = \mathbf{\cK} \by,
\end{equation}
where $\cK$ is a Koopman operator. However, determining those observers is a challenge.  In this paper, these observers are referred to as embeddings. 

In this work, we harness the capabilities of deep learning to identify finite-dimensional Koopman embeddings whose dynamics can be described by linear systems. In other words, we aim to approximate the infinite-dimensional Koopman operator by a finite-dimensional one. 
While several deep learning approaches have been explored in related contexts, none have specifically addressed canonical Hamiltonian systems. To that end, we aim to design embeddings $\by = \boldsymbol{\phi}(\bx)$, where $\boldsymbol{\phi}$ can be regarded as construction of embeddings based on the measurement $\bx$ so that their dynamics can be given by a linear system of the form \eqref{eq:HamiltonianEquations}.
Furthermore, it is crucial to enforce the requirement that the mapping of measurements to embeddings is a symplectic transformation. 
%
To accomplish these goals, the following objectives are pursued. 
\begin{itemize}
	\item \textbf{Learning linear representation.} Our goal is to learn $2m$-dimensional embeddings $\by$ from the measurements $\bx$ using a mapping function $\boldsymbol{\phi}$, i.e., $\by = \boldsymbol{\phi}(\bx)$, and $\bx$ can be recovered using $\bx = \boldsymbol{\phi}^{-1}(\by)$. It is achieved by means of an autoencoder, where $\boldsymbol{\phi}$ acts as an encoder that takes $\bx$ to $\by$, and $\boldsymbol{\psi} := \boldsymbol{\phi}^{-1}$ behaves like a decoder that takes  $\by$ back to $\bx$. The quality of the autoencoder is evaluated by the loss as follows:
	\begin{equation}\label{eq:loss_autoencoder}
		\mathcal L_{\text{encdec}}= \left\| \bx-\boldsymbol{\psi}(\boldsymbol{\phi}(\bx))\right\|.
	\end{equation}
	\item \textbf{Linear dynamics.} To ensure linear dynamics of the embeddings $\by$, we require that $\by$ satisfies \eqref{eq:HamiltonianEquations}. Thus, there exists a Hamiltonian for $\by$, which takes a quadratic form, i.e., $\bH^{\left(\mathrm{L}\right)}(\by) = \by^\top \bA \by + \bb^\top \by + \bc$. As a result, we can write the dynamics of $\by$ as follows:
	\begin{equation}
		\ddt {\by}(t) = \bJ_{2m}\nabla_\by\bH^{\left(\mathrm{L}\right)}(\by),
	\end{equation}
	where $\bJ_{2m}$ is a symplectic matrix, and $\nabla_\by \bH^{\left(\mathrm{L}\right)}(\by)$ represents the gradient of the Hamiltonian with respect to $\by$.	This goal is assessed using the following objective function:
	\begin{equation}\label{eq:loss_deri}
		\mathcal L_{\text{deri}} = \left\|\nabla_\bx\boldsymbol{\phi}(\bx)\dot\bx - \bJ_{2m}\nabla_\by\bH^{\left(\mathrm{L}\right)}(\boldsymbol{\phi}(\bx))\right\|.
	\end{equation}	
	\item \textbf{Symplectic transformation.} Finally, we note that the mapping $\boldsymbol{\phi}$ that takes $\bx$ to $\by$ is required to be a symplectic transformation. Therefore, for $m \geq n$, we enforce it by using the following condition:
	\begin{equation}\label{eq:loss_symplectic}
		\mathcal L_{\text{symp}} = \left\| (
		\Diff \boldsymbol{\phi}_\bx)^\top \bJ_{2m} \Diff\boldsymbol{\phi}_\bx - \bJ_{2n}\right\|.
	\end{equation}
\end{itemize}
We consider $\|\cdot\|$ to be the mean-squared error. Using a weighted sum, we combine these three losses, given in \cref{eq:loss_autoencoder,eq:loss_deri,eq:loss_symplectic}, thus resulting in the total loss as follows:
\begin{equation}\label{eq:total_loss}
	\mathcal L = \lambda_1\mathcal L_{\text{encdec}} +\lambda_2 \mathcal L_{\text{symp}}  +\lambda_3\mathcal L_{\text{deri}} ,
\end{equation}
where $\lambda_{\{1,2,3\}}$ are hyper-parameters. Consequently, we can expect to learn the Koopman embeddings with symplectic properties from the given measurements $\bx$.

\paragraph{Stability guarantee.} Stability is a crucial property for differential equations. Thus, we next discuss the stability of the learned dynamical system for the identified embeddings $\by$. With this regard, we aim to study the stability of dynamical systems of the form:
\begin{equation}\label{eq:obs_can_sys}
	\dot{\by}(t) = \bJ_{2m}\nabla_\by\bH^{\left(\mathrm{L}\right)}(\by),
\end{equation}
where $\bH^{\left(\mathrm{L}\right)}(\by) = \by^\top \bA \by + \by^\top \bb + \bc$. For the stability of the system \eqref{eq:obs_can_sys}, it is sufficient to ensure that $\bH^{\left(\mathrm{L}\right)}(\by)$ is bounded from below and radially unbounded. This means that 
\begin{equation}\label{eq:hamiltonian_quad_properties}
 	\bH^{\left(\mathrm{L}\right)}(\by) \geq \gamma,\quad \text{and}\quad \lim_{\|\by\|\rightarrow \infty}\bH^{\left(\mathrm{L}\right)}(\by) \rightarrow \infty, \qquad \forall \by \in \R^{2m},
\end{equation}
where $\gamma \in \R$ is a constant and finite.  Since a constant in $\bH^{\left(\mathrm{L}\right)}$ does not affect the dynamics for $\by$, we can assume $\gamma$ to be zero without loss of generality. Next,  we discuss a construction of  Hamiltonian functions which are quadratic and satisfy \eqref{eq:hamiltonian_quad_properties}. 

\begin{theorem}\label{thm:stable_ham_linear}
	Consider a Hamiltonian $\bH^{\left(\mathrm{L}\right)}(\by) = \by^\top \bA \by + \by^\top \bb + \bc$ which can be written as 
	\begin{equation}
	\bH^{\left(\mathrm{L}\right)}(\by) = \begin{bmatrix}	\by \\ w \end{bmatrix}^\top \bQ \begin{bmatrix}	\by \\ w \end{bmatrix},
	\end{equation}
	where $w$ is a scalar constant, and  $\bQ = \bQ^\top > 0$. Then,
	\begin{enumerate}[label=(\alph*)]
		\item $\bH^{\left(\mathrm{L}\right)}(\by)$ is bounded from below, and
		\item $\bH^{\left(\mathrm{L}\right)}(\by)$ is radially unbounded, i.e., $\lim\limits_{\|\by\|\rightarrow \infty}\bH^{\left(\mathrm{L}\right)}(\by)~\rightarrow~\infty$.
		\item Furthermore, we have $\|\by(t)\|_2^2 ~\leq~ \dfrac{\bH^{\left(\mathrm{L}\right)}(\by_0)}{\sigma_{\min}\left(\bQ\right)},$~$\forall t \geq 0$, where $\by(t)$ is the solution of \eqref{eq:obs_can_sys} at time $t$ for a given initial condition $\by_0$, and $\sigma_{\min}(\cdot)$ denotes the smallest singular values. 
	\end{enumerate}
\end{theorem}
\begin{proof} 
	
	\begin{enumerate}[label=(\alph*)]
		\item Since $\bH^{\left(\mathrm{L}\right)}(\by)$ exhibits a sum-of-squares form, it is non-negative by construction. Hence, it is bounded from below, and precisely, $\bH^{\left(\mathrm{L}\right)}(\by) \geq 0$. 
		\item Using  the symmetric positivity property of $\bQ$, we have
		\begin{align*}
		\bH^{\left(\mathrm{L}\right)}(\by)  \geq \sigma_{\min}(\bQ)   \left\| \begin{bmatrix}	\by(t) \\ w \end{bmatrix} \right\|_2^2 \geq \sigma_{\min}(\bQ)  \|\by(t)\|_2^2. 
		\end{align*}
		Since $\sigma_{\min}(\bQ) > 0$, it is clear that as $\|\by(t)\|_2 \rightarrow \infty$, $\bH^{\left(\mathrm{L}\right)}(\by) \rightarrow \infty$. Hence, $\bH^{\left(\mathrm{L}\right)}(\by)$ is radially unbounded. 
		\item The dynamics of $\by$ is given by a canonical Hamiltonian system of the form \eqref{eq:obs_can_sys}. For this, we know that the Hamiltonian is constant along the trajectory, meaning 
		\[\bH^{\left(\mathrm{L}\right)}(\by_0 ) = \bH^{\left(\mathrm{L}\right)}(\by(t)), \quad \forall t\geq 0,\]
		where $\by_0$ is an initial condition and $\by(t)$ is the solution of \eqref{eq:obs_can_sys} at time $t$ for the given initial condition $\by_0$. Thus, we have
		\begin{align*}
		\bH^{\left(\mathrm{L}\right)}(\by_0) &= \bH^{\left(\mathrm{L}\right)}(\by(t)) \\
		& = \begin{bmatrix}	\by(t) \\ w \end{bmatrix}^\top \bQ \begin{bmatrix}	\by(t) \\ w \end{bmatrix} 
		\geq \sigma_{\min}(\bQ) \left\|\begin{bmatrix}	\by(t) \\ w \end{bmatrix}\right\|_2^2
		\geq \sigma_{\min}(\bQ) \left\|	\by(t) \right\|_2^2.
		\end{align*}
		Based on it, we immediately have the result (c), which concludes the proof.
	\end{enumerate}
\end{proof}
Utilizing the result of \Cref{thm:stable_ham_linear}, we can parameterize the Hamiltonian function for the observers $\by$ to guarantee the stability of their time-evaluation. 
\subsection{Systems with continuous spectra and remedy by cubicized representations}\label{subsec:cubic_stability}
Using the Koopman operator theory, one can strive to achieve a global linearization of nonlinear dynamical systems. However, in the case of certain nonlinear systems, such as the nonlinear pendulum, continuous spectra may be exhibited, making it difficult to learn a global linearization.  Capturing the dynamics of such systems requires a high-dimensional embedding space to account for the continuous spectra.
In a recent work \cite{lusch2018deep}, the eigenvalues of the linear Koopman operator are parameterized using a neural network that depends on the embeddings themselves, enabling richer dynamics to capture. However, it is then no longer a purely linear operator.

In this work, we instead take a different approach by exploring learning simpler weak nonlinear systems rather than  linearized Koopman embeddings. It builds on the theory, articulating that smooth nonlinear dynamical systems can be adequately represented as polynomial systems using a finite number of observers \cite{savageau1987recasting,morGu11}. This is in contrast to the Koopman operator theory, which may require infinite-dimensional embeddings to represent nonlinear systems as linear, particularly those with continuous spectra. Furthermore, nonlinear polynomial systems can be written as quadratic or cubic systems with a few more, but still finite, observers. 
Based on these observations, the authors in \cite{yildizetal23} investigated a problem of learning representations for nonlinear canonical Hamiltonian systems with cubic polynomial Hamiltonian. However, it can be noticed that nonlinear canonical Hamiltonian systems with cubic Hamiltonian cannot be stable. 
This instability arises from the nature of the cubic polynomial functions, which are neither bounded from below nor from above.
Consequently, even though the dynamics may exhibit the Hamiltonian conservation along the trajectories, the system can still be inherently unstable. An illustrative example is given in \Cref{exm:cubicHamiltonian}.

To address these drawbacks, namely with continuous spectra and Hamiltonian systems with cubic Hamiltonian, we seek to identify suitable embeddings from the measurements so that the dynamics of the embeddings can be given by a canonical Hamiltonian system of the form \eqref{eq:HamiltonianEquations}, where the Hamiltonian is a quartic polynomial function, instead of being cubic. Consequently, we expect to design finite-dimensional embeddings that can explain the dynamics of nonlinear Hamiltonian systems with continuous spectra while ensuring the dynamic stability of the embeddings. To achieve this, it is essential for the underlying Hamiltonian to be radially unbounded and non-negative. We delve into this topic in the following discussion.

\begin{theorem}\label{thm:stable_ham}
	Consider a canonical Hamiltonian system of the form \eqref{eq:HamiltonianEquations} as:
	\begin{equation}\label{eq:obs_can_quartic}
	\dot{\by}(t) = \bJ_{2m}\nabla_\by\bH^{\left(\mathrm{C}\right)}(\by),
	\end{equation}
	 where $\bH^{\left(\mathrm{C}\right)}(\by)$ is a quartic Hamiltonian as follows:
	$$\bH^{\left(\mathrm{C}\right)}(\by) = \ba_0 + \ba_1^\top\by + \ba_2^\top \left(\by\otimes \by\right)+ \ba_3^\top \left(\by\otimes \by\otimes \by\right) + \ba_4^\top \left(\by\otimes \by\otimes \by\otimes \by\right),$$ where $\{\ba_1,\ba_2,\ba_3,\ba_4\}$ are vectors of appropriate sizes and $\otimes$ denotes the Kronecker product. If $\bH^{\left(\mathrm{C}\right)}(\by)$ can be written as a sum-of-square, i.e., 
	\begin{equation}\label{eq:hamiltonain_quartic}
	\bH^{\left(\mathrm{C}\right)}(\by) = \begin{bmatrix} \by \\ \by \otimes \by \\ w \end{bmatrix}^\top \bQ \begin{bmatrix} \by \\ \by \otimes \by \\ w \end{bmatrix},
	\end{equation}
	where $w$ is a scalar constant and $\bQ = \bQ^\top > 0$. Then, the following hold:
	\begin{enumerate}
		\item $\bH^{\left(\mathrm{C}\right)}(\by)$ is bounded from below, 
		\item $\bH^{\left(\mathrm{C}\right)}(\by)$ is radially unbounded, i.e., $\lim\limits_{\|\by\|\rightarrow \infty}\bH^{\left(\mathrm{C}\right)}(\by)~\rightarrow~\infty$, and
		\item $\|\by(t)\|^2_2 + \|\by(t)\|^4_2~\leq ~\dfrac{\bH^{\left(\mathrm{C}\right)}(\by_0)}{\sigma_{\min}\left(\bQ\right)}$, where $\by(t)$ is the solution of \eqref{eq:obs_can_quartic} at time $t$ for a given initial condition $\by_0$.
\end{enumerate}
\end{theorem}
\begin{proof}
	The proofs of (a) and  (b) follow the same arguments given in \Cref{thm:stable_ham_linear}; thus, we will omit them here. 
	For (c), we note that 
	\[\bH^{\left(\mathrm{C}\right)}(\by_0 ) = \bH^{\left(\mathrm{C}\right)}(\by(t)), \quad \forall t \geq 0,\]
	where $\by(t_0)$ is an initial condition and $\by(t)$ is the solution of \eqref{eq:obs_can_quartic} at time $t$ for the given initial condition $\by_0$. Thus, we have
	\begin{align*}
	\bH^{\left(\mathrm{C}\right)}(\by_0) &= \bH^{\left(\mathrm{C}\right)}(\by(t) )\\
	& = \begin{bmatrix}	\by(t) \\ \by(t) \otimes \by(t) \\ w \end{bmatrix}^\top \bQ \begin{bmatrix}	\by(t) \\ \by(t) \otimes \by(t) \\ w \end{bmatrix} 
	\geq \sigma_{\min}(\bQ) \left\|\begin{bmatrix}	\by(t) \\ \by(t) \otimes \by(t) \\ w \end{bmatrix}\right\|^2_2 
	\geq \sigma_{\min}(\bQ) \left(\left\|	\by(t) \right\|_2^2 +  \left\|	\by(t) \right\|_2^4\right).
	\end{align*}
	In the above equation, we have used the relation $ \|\by\otimes \by\|_2^2 = \|\by\|^4_2$. This completes the proof.
\end{proof}
The positive definite condition on the matrix $\bQ$  is sufficient for stability; in the following, we discuss a case where $\bQ$ can be semi-definite, yet the Hamiltonian system is stable. 
\begin{lemma}\label{lemma:Q_DSP}
	Consider a canonical Hamiltonian system with a Hamiltonian 
	\begin{equation}\label{eq:hamiltonain_quartic_spd}
		\bH^{\left(\mathrm{C}\right)}(\by) = \begin{bmatrix} \by \\ \by \otimes \by \\w \end{bmatrix}^\top \bQ \begin{bmatrix} \by \\ \by \otimes \by \\w \end{bmatrix},
	\end{equation}
	where $\by\in \R^{2m}$ and $\bQ = \bQ^\top \geq 0$. Let $\bV \in \R^{\hat{m},\tm}$, where $\tilde m = 2m + 4m^2 + 1$ and $\hat{m} < \tilde m$, be a matrix so that 
	\begin{equation}\label{eq:decompose_Q}
		\bQ = \bV^\top\bQ_1 \bV,
	\end{equation}
	where $\bQ_1 = \bQ_1^\top > 0 $.	If $\left\|\bV\begin{bmatrix} \by \\ \by \otimes \by \\ w \end{bmatrix}\right\|_2^2 \geq \bg(\by)$, where the function $\bg(\by)$ is radially unbounded, then the system is stable. 
\end{lemma}
\begin{proof}
	First, note that for stability, we require the underlying Hamiltonian to be bounded from below and radially unbounded. Since $\bH^{\left(\mathrm{C}\right)}$ has a sum-of-squares form, it is bounded from below by construction. For radially unboundedness, we substitute for $\bQ$ from \eqref{eq:decompose_Q} in \eqref{eq:hamiltonain_quartic_spd}, yielding
	\begin{equation}
		\bH^{\left(\mathrm{C}\right)}(\by) = \begin{bmatrix} \by \\ \by \otimes \by \\w \end{bmatrix}^\top \bV^\top \bQ_1 \bV \begin{bmatrix} \by \\ \by \otimes \by \\w \end{bmatrix} \geq 
		\sigma_{\min}(\bQ_1) \left\|\bV \begin{bmatrix} \by \\ \by \otimes \by \\w \end{bmatrix} \right\|_2^2 = \sigma_{\min}(\bQ_1)\bg(\by).
	\end{equation} 
	Since the function $\bg$ is assumed to be radially unbounded, the Hamiltonian is also radially unbounded, thus concluding the proof.
\end{proof}
An illustrative example (see \Cref{exa:spd_Q}) demonstrates  \Cref{lemma:Q_DSP}. Thus, we focus on the quartic hypothesis applied to the Hamiltonian, which can lead to the desired embedding- stable global cubicized embeddings for nonlinear canonical Hamiltonian dynamical systems. To achieve this, we follow similar principles discussed in the previous subsection, where we discussed learning linearized Koopman embeddings for nonlinear canonical Hamiltonian systems. However, the crucial distinction lies in the hypothesis about the Hamiltonian function, which now takes the form of a quartic function (as given in  \eqref{eq:hamiltonain_quartic}), rather than a quadratic one. 
\begin{remark}
	In \Cref{thm:stable_ham} and \Cref{lemma:Q_DSP}, we utilized a general quartic function in the sum-of-squares form. However, we can relax it by using $\by\circ\by$ instead of $\by\otimes\by$, where $\circ$ denotes the Hadamard product. This modification allows us to reduce the number of parameters defining $\bH^{\left(\mathrm{C}\right)}$ while still retaining a quartic function in a sum-of-squares form. Although this relaxation results in a slightly less general quartic function, we empirically observe in our numerical experiments that it is less prone to over-fitting; hence, we utilize this in our experiments.
\end{remark}

\input{numerics_section}

\section{Conclusions}\label{sec:conclusions}
To conclude, we have explored the capabilities of deep learning to learn embeddings that aim to provide a universal linearized framework for nonlinear canonical Hamiltonian systems. It can be viewed as a structure-preserving finite-dimensional Koopman operator learning for nonlinear canonical Hamiltonian systems.
Designing embeddings with the desired properties is inherently challenging. To address this, we employed the autoencoder framework, which proved to be instrumental in achieving our objectives.
In addition, we have overcome the challenges associated with learning global linearization for nonlinear systems featuring a continuous eigenvalue spectrum by learning a global cubicized framework. 
It is inspired by the lifting principle \cite{savageau1987recasting,morGu11}, allowing us to express nonlinear systems as cubic systems in a finite-dimensional embedding. 
Furthermore, we delved into how to enforce stability with regard to the dynamics of the learned embeddings, and it is achieved by exploring a sum-of-squares formulation for the underlying Hamiltonian. We have demonstrated the efficiency of the proposed methodologies for low-dimensional cases. %
We additionally have discussed a way to deal with high-dimensional data. To handle such cases, we utilized proper-orthogonal decomposition (POD) to derive a low-dimensional representation of these high-dimensional data. We then leveraged this representation to learn a suitable embedding with the desired properties.
Moreover, from these low-dimensional representations obtained using POD, we investigated various decoder approaches, namely, linear-decoder, quadratic-decoder, and convolutional neural network-based decoder, to reconstruct the solutions on the full spatial domain. Notably, we observed that convolutional neural networks can perform well when the POD basis does not capture enough of the energy present in the training data. Hence, with a nonlinear decoder based on a convolutional neural network, a good reconstruction with only a few POD modes can be done for systems with a slow decay of the Kolmogorov $N$-width. 
Lastly,  we emphasize that incorporating prior knowledge about the canonical Hamiltonian framework played a crucial role in enabling us to learn dynamics in low-data regimes. It also enhances the interpretability of the learned coordinate system in the sense that it obeys the properties of the canonical Hamiltonian system by design. 

This work opens up several challenging and important research avenues for the future. The foremost challenge is to design an autoencoder architecture and to determine the dimension of the coordinate system $\by$. Automating the design of the autoencoder architecture using neural architecture search, see, e.g., \cite{elsken2019neural}, would be highly beneficial, as well as exploring methods to determine the optimal dimensionality of the coordinate system. 
Additionally, in practical scenarios, the presence of noise in data is inevitable. Therefore, it will be crucial to adapt the proposed methodologies to handle such cases, and for this, one can employ the principles discussed, e.g., in \cite{rudy2019deep,goyal2022neural}. Furthermore, for high-dimensional data, we have determined a low-dimensional representation of these data by using POD, which can be seen as a linear projection or transformation. However, an investigation based on convolution neural networks to determine a low-dimensional projection would be worthwhile.
In the future, we also aim to apply these techniques to practical problems, particularly those arising in the study of planetary motion. 

\addcontentsline{toc}{section}{References}
\bibliographystyle{ieeetr}
\bibliography{mybib,ref}

\appendix
\section{Demonstrating Examples}
\begin{example}[Cubic Hamiltonian]\label{exm:cubicHamiltonian}
	Consider a nonlinear system with the Hamiltonian $\mathcal H (q,p)= \frac{p^2}{2}+\frac{q^2}{2} +\frac{q^3}{3}$. The associated governing equations for the oscillator are thus given by
	\begin{equation}\label{eqn:cubic_sys}
		\begin{aligned}
			&\dot{q}=p \\
			&\dot{p}=-q - q^2.
		\end{aligned}
	\end{equation}
	It is easy to see that the system \eqref{eqn:cubic_sys} is not globally stable. For large negative values for $q$ and $p$, it can be observed that $p \rightarrow -\infty$ and $q \rightarrow -\infty$ as $t\rightarrow \infty$, despite $\cH(q,p)$ being constant. Its stability can also be argued by the fact that the Hamiltonian is neither bounded from below nor from above. 
\end{example}

\begin{example}\label{exa:spd_Q}
	Consider a system with a Hamiltonian $\cH(p,q) = p^2 + q^4 $. If we aim to write the Hamiltonian in the form given in \eqref{eq:hamiltonain_quartic_spd}, then we have
	\begin{equation}
		\cH(p,q) = \begin{bmatrix} q\\ p\\ q^2\\ qp \\ pq \\ p^2 \end{bmatrix}^\top  \underbrace{\begin{bmatrix} 0 & 0& 0& 0& 0 & 0 \\ 0 & 0& 0& 0& 0 & 0 \\ 0 & 0& 1& 0& 0 & 0\\ 0 & 0& 0& 0& 0 & 0\\ 0 & 0& 0& 0& 0 & 0\\0 & 0& 0& 0& 0 & 1 \end{bmatrix}}_{\bQ}  \begin{bmatrix} q\\ p\\ q^2\\ qp \\ pq \\ p^2 \end{bmatrix}
	\end{equation}
	It can be noted that $\bQ = \bQ^\top \geq 0$. To decompose $\bQ$ to write in the form given in \eqref{eq:decompose_Q}, we define 
	\begin{equation}
		\bV = \begin{bmatrix} 0 & 0& 1& 0& 0 & 0 \\0 & 0& 0& 0& 0 & 1 \\ 1 & 0& 0& 0& 0 & 0 \\ 0 & 1& 0& 0& 0 & 0 \\  0 & 0& 0& 1& 0 & 0\\ 0 & 0& 0& 0& 1 & 0 \end{bmatrix}.
	\end{equation}
	First, note that $\bV^\top\bQ \bV = \begin{bmatrix}
		\bQ_1 & \mathbf{0} \\  \mathbf{0} &  \mathbf{0}
	\end{bmatrix}$, where $\bQ_1 = \bI_2$ and $ \mathbf{0}$ is a zero-matrix of the appropriate size. Furthermore, notice that 
	$\left\|\bV\begin{bmatrix} \bz \\ \bz \otimes z \end{bmatrix}\right\|_2 = q^4 +  p^2 =: \bg(\bz)$, where $\bz = [q,p]$. It can be noted that $\bg(\bz)$ is non-negative, and for $\|\bz\| \rightarrow \infty$, $\bg(\bz) \rightarrow \infty$. Since $\bg(\bz)$ is finite constant (as it is a Hamiltonian, corresponding to the underlying dynamical system), $\|\bz\|$ cannot be infinite; otherwise, $\bg(\bz)$ must also be infinite. Using these arguments, the dynamics will be stable, and $\|\bz(t)\|_2$ will be bounded $\forall t \geq 0$. 
\end{example}

\section{Training Set-up}\label{appendix:training}
First note that all the experiments are done using \texttt{PyTorch} on a machine with an \texttt{Intel\textsuperscript{\tiny\textcopyright} Core\textsuperscript{\tiny TM} i5-12600K} CPU and \texttt{NVIDIA RTX\textsuperscript{\tiny TM} A4000(32GB)} GPU.
We set the number of epochs for each example to $4000$ and the initial learning rate to $3\cdot 10^{-3}$. We have used the Adam optimizer \cite{kingma2014adam} to train the parameters for the autoencoder and for the Hamiltonian, which defines the dynamics of the learned coordinate system or embedding. Furthermore, we utilize a step decay for the learning rate, using the \texttt{StepLR} implementation in \texttt{PyTorch} with a step of $0.1$.   For each experiment, we set the hyper-parameters $(\lambda_{1}, \lambda_{2}, \lambda_{3})$ in \eqref{eq:total_loss} to $(0.1,1.0,1.0)$. 
For the autoencoder architecture, we use a multi-layer perceptron (MLP) architecture with three hidden layers and SeLU activation, and their number of neurons are given in \Cref{tab:hyperparameters-low}. The table also contains additional hyper-parameters for our illustrative examples. Note that all these parameters are derived based on a combination of experience and quick parameter grid-search, but a systematic and automatic way to determine them remains a future research topic. 
Furthermore, we control the parameters, defining the Hamiltonian for the coordinate system, by penalizing using the weighted $L_1$-norm, and the weight for all examples is set to $10^{-4}$.

\begin{table}[tb]
	\renewcommand{\arraystretch}{1.25}
	\caption{The table contains all the hyper-parameters to learn the dynamics of the considered examples. We note that $\texttt{wd}^{(a)}$ and  $\texttt{wd}^{(h)}$, respectively, denote weight decay parameters for the learnable parameters for the autoencoder and Hamiltonian for the transformed coordinate system. }
	\label{tab:hyperparameters-low}
	\begin{tabular}{|c|c|c|c|c|}
		\hline
		{\bfseries Examples/Parameters}                                                              & \begin{tabular}[c]{@{}c@{}}{\bfseries 		Encoder layers} \\ {\bfseries  [neurons]}\end{tabular} & \begin{tabular}[c]{@{}c@{}}{\bfseries  Coordinate} \\ {\bfseries dimension}\end{tabular} & \begin{tabular}[c]{@{}c@{}}{\bfseries Batch} \\ {\bfseries size}\end{tabular} & \begin{tabular}[c]{@{}c@{}}{\bfseries weight decay} \\ {\bfseries  ($\texttt{wd}^{(a)}, \texttt{wd}^{(h)}$) }\end{tabular}
		\\ \hline
		Nonlinear pendulum  & $[8,8,8]$    & 2  & $ 32$ & $(10^{-5},10^{-5})$  \\ \hline
		Nonlinear oscillator& $ [8,8,8]$  & 2  & $ 32$ & $(10^{-5},10^{-5})$  \\ \hline
		Lotka-Volterra  & $ [8,8,8]$  & $4$& $64$  & $\left(10^{-5},10^{-4}\right)$ \\ \hline
		Nonlinear Schrodinger& $ [12,12,12]$    &    4& $ 32$  &        $(10^{-5},10^{-3})$ \\ \hline
		Linear Wave & $ [12,12,12]$ & 6 &  $ 32$ &  $(10^{-5},10^{-5})$  \\ \hline
	\end{tabular}
\end{table}

\section{Quadratic-Decoder and Convolutional Neural Network-Based Decoder}\label{appendix:decoder}
Here, we discuss the construction of solutions on the designated spatial domain using the  low-dimensional POD coordinates. A simple approach is a linear decoder in which the POD basis are used to reconstruct the solution. In the following, we discus quadratic-decoder and convolutional neural network (CNN) based decoder. 
\paragraph{Quadratic-decoder:} Following \cite{jain2017quadratic,geelen2023operator}, we pose the following optimization problem, aiming to reconstruct the solution on the entire domains:
\begin{equation}
	\min_{\bV, \bH} \sum_{i=1}^\cN \|\bx_i - \bV\by_i - \bH\left(\by_i\otimes \by_i\right)\|,
\end{equation}
where $\bx_i$ and $\by_i$ are the $i$th instance of the full domain solution and the corresponding POD coordinates, respectively and the norm $\|\cdot\|$ in the above equation is a combination of Frobenius norm and $L_1$-norm with an equal weight-age. Since a quadratic function is to reconstruct the full-domain solution, we refer to it as a quad-decoder.  

\paragraph{CNN based decoder:} Benefiting with the success of deep learning, we also propose to reconstruct the full-domain solutions by means of CNNs. It takes the POD coordinates as inputs and provides an output as the full-domain solution. Precisely, we set up the following optimization problem:
\begin{equation}
	\min_\eta \sum_{i=1}\|\bx_i - \cN_\eta(\by_i)\|,
\end{equation}
where $\cN$ is a CNN, and its trainable parameters are denoted by $\eta$, and the norm $\|\cdot\|$ in the above equation is a combination of Frobenius norm and $L_1$-norm with an equal weight-age. We propose the architecture of a CNN  as in \Cref{fig:cnn_decoder}.
\begin{figure}
	\centering
	\includegraphics[width=0.75\linewidth]{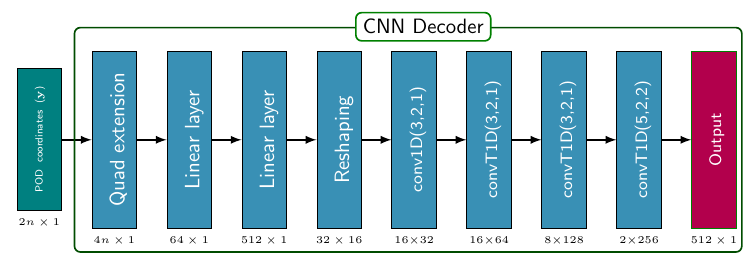}
	\caption{The figure summarises the architecture of the CNN decoder. In the figure, the quad-extension block concatenates the input and its Hadamard product, i.e., $[\by,\by\circ \by]$, and
		convT1D$(k,s,p)$ is a 1D transpose convolution layer with transpose kernel size $k$, stride size $s$, padding size $p$. We have used an output padding size of $1$ to obtain a symmetric auto-encoder structure in 1D transpose convolution layers. We denote the size of the output block below each block. }
	\label{fig:cnn_decoder}
\end{figure}
The parameters of these two decoders are optimized using the Adam optimizer for $600$ epochs. We also set the initial learning rate to $10^{-3}$, which is reduced by a factor of ten after each $250$ epochs. We also use weight-decay with a weightage of $10^{-5}$ to avoid over-fitting. Moreover, we have set the batch-size as $32$.

%

\end{document}

%% file: numerics_section.tex
\section{Results for Low-dimensional Nonlinear Systems}\label{sec:num}	
In this section, we begin by examining the performance of the proposed methodologies for low-dimensional dynamical systems. We refer to the methods presented in \Cref{subsec:linear_stability,subsec:cubic_stability} as \linearembs\ and \cubicembs, respectively, where `\texttt{s}' at the beginning indicates the stability guarantee. 
We consider three examples to evaluate their performance: the simple nonlinear pendulum, a harmonic oscillator, and the Lotka-Volterra equation. 
We compare our proposed methodologies with the one discussed in \cite{yildizetal23}, which is also purely data-driven and aims to learn a quadratic representation for canonical Hamiltonian nonlinear systems. We refer to it as  \texttt{quad-embs}.
Furthermore,  It is worth noting that all three methods considered involve auto-encoders based on neural networks. Therefore, we employ the same neural network architectures for all methods. Additionally, the training settings remain consistent across all three methods, which are discussed in detail in \Cref{appendix:training}.
Lastly, we utilize the implicit midpoint rule as the time integrator that preserves the symplectic structure after time discretization; see \cite{hairer2006structure}.
%

\subsection{Nonlinear pendulum} 
In the first example, we consider an ideal nonlinear pendulum in a normalized form, whose governing equations are given by 
\begin{equation}\label{eq:nonlinearpendulum}
\begin{bmatrix}
\dot{p}(t) \\ \dot{q}(t)
\end{bmatrix} = \begin{bmatrix}
-\sin(q(t))\\ p(t)
\end{bmatrix},
\end{equation}
where $q$ and $p$ denote the position and momentum related to the pendulum dynamics. Moreover, the Hamiltonian of \eqref{eq:nonlinearpendulum} is given by
\begin{equation}
\Hamiltonian(q,p) = (1-\cos (q)) + \dfrac{1}{2}p^2,
\end{equation}

To generate the training dataset, we follow \cite{yildizetal23}. The initial values of $q$ and $p$ are chosen within the range of $[-3, 3] \times [-3, 3]$; however, to avoid a $360$-degree swing of the pendulum, we choose them such that the energy $\Hamiltonian(q,p) \leq  2$. We consider $20$ random initial conditions, and for each initial condition, we take $25$ equidistant data points within the time interval $[0,20]$. 

\begin{figure}[tb]
	\includegraphics[width=0.9\linewidth]{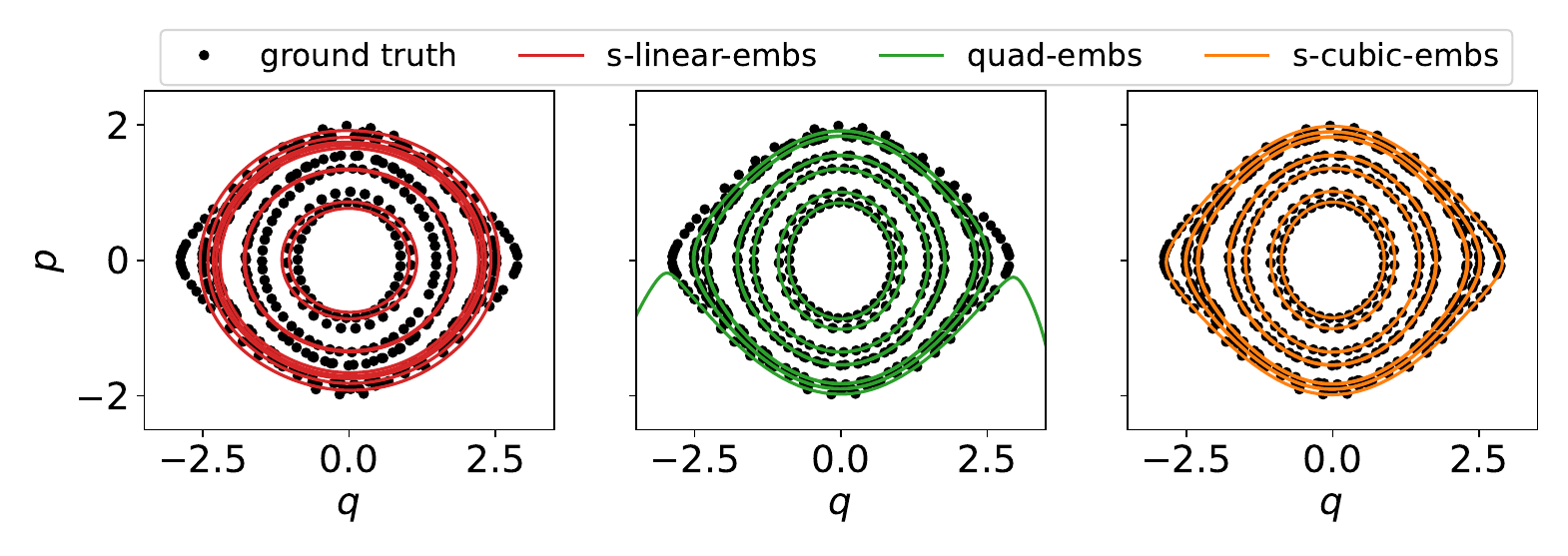}
	\caption{Nonlinear pendulum: The figure shows a comparison of various learned models with the ground truth in the phase space for test initial conditions.}
	\label{fig:pend-phase}
\end{figure}

Next, we learn the desired embeddings and their dynamics using \linearembs, \quadembs, and \cubicembs. To accomplish this, we fix the dimension of the embeddings to two and identify them by means of an auto-encoder for each of the three methods. The autoencoder architecture and training setup details can be found in \Cref{appendix:training}.
After learning the desired embeddings, we evaluate the performance of the learned representation using $25$  test initial conditions that were not part of the training set. In addition, for the testing phase, we consider a much longer time interval of $[0,50]$, compared to the training time interval $\left([0,20]\right)$,  and measure the performance using $2500$ points within the testing interval. 

First, we illustrate the quality of the learned $q$ and $p$ using the embeddings in phase space, shown in \Cref{fig:pend-phase}. These figures clearly demonstrate that \linearembs\ fails to capture the dynamics accurately. This is potentially attributed to the fact that the nonlinear pendulum has continuous spectra, making it difficult to learn a finite-dimensional universal linear representation. Furthermore, \quadembs\ learns the dynamics well for small amplitudes $q$ and $p$.
However, it struggles to capture the dynamics for larger amplitudes, and its limitations become even more apparent when dealing with highly nonlinear dynamics, where it fails completely despite preserving the Hamiltonian for the embeddings.
On the other hand, the proposed methodology, namely \cubicembs, accurately produces the dynamics for all test initial conditions. Notably, all these dynamics are stable by construction, making \cubicembs\ a reliable and robust approach even for nonlinear systems with continuous spectra. 
\begin{figure}[tb]
	\includegraphics[width=0.350\linewidth]{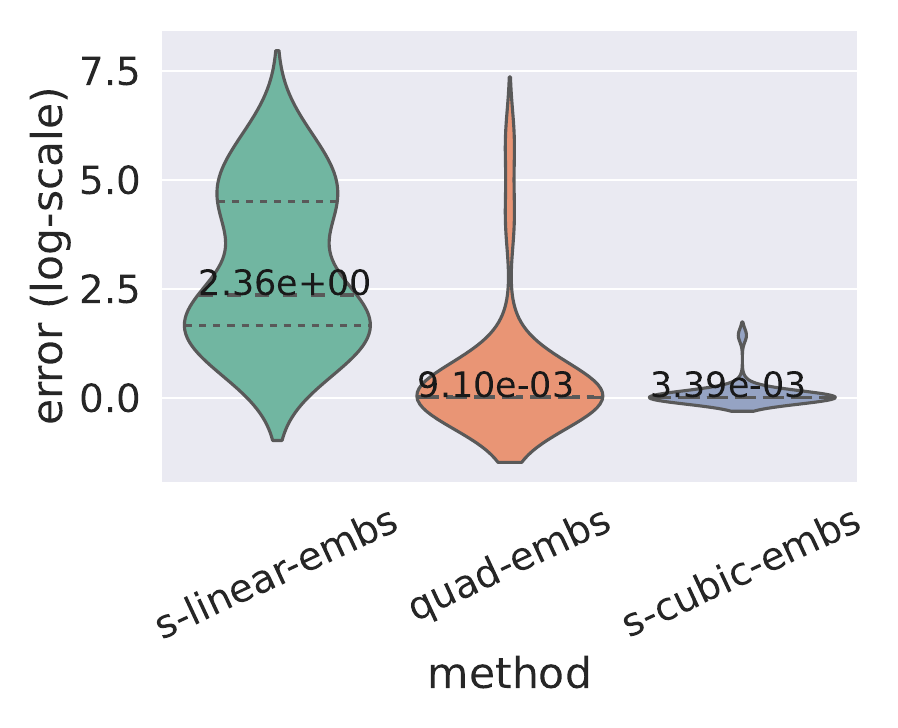}
	\caption{Nonlinear pendulum:  A qualitative performance analysis of different methods for test initial conditions. The middle dashed line shows the median, with the precise number, across all test initial conditions.} 
	\label{fig:pendulum_voilin}
\end{figure}

For further qualitative analysis, we employ the following error measure:
\begin{equation}\label{eq:error_measure}
\cE(\bx_0) = \left.\dfrac{1}{\cN \cdot m}\sum_{i=1}^\cN \|\bx^{\texttt{gt}}(t_i;\bx_0) - \bx^{\texttt{pred}}(t_i;\bx_0) \|_2^2\right.,
\end{equation}
where $\bx^{\texttt{gt}}(t_i;\bx_0)\in \R^m$ and $\bx^{\texttt{pred}}(t_i;\bx_0)\in \R^m$ are, respectively,  the ground truth and predicted solutions at time $t_i$ for a given initial condition $\bx_0$. Using the measure \eqref{eq:error_measure}, we calculate $\cE$ for $25$ random test initial conditions and present a violin plot in \Cref{fig:pendulum_voilin}. The plot clearly indicates the superior performance of \cubicembs\ compared to the other two considered methods. In \Cref{fig:pendulum_timesimulation}, we additionally plot time-domain simulations for two specific cases from the $25$ test conditions for which \cubicembs\ performs the best and worst with respect to the ground truth. These figures show that for the best case, \quadembs\ also does well to capture the dynamics, but \linearembs\ fails to do so. Conversely, for the worst case, only \cubicembs\ could accurately capture the dynamics,  although a slight shift in behavior is observed.

\begin{figure}[tb]
	\centering
	\begin{subfigure}[b]{1\textwidth}
		\includegraphics[width=0.24\linewidth]{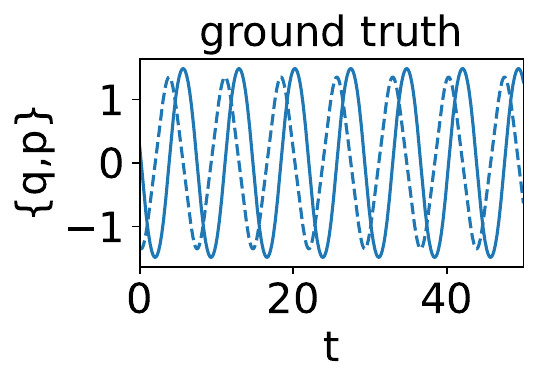}				\includegraphics[width=0.24\linewidth]{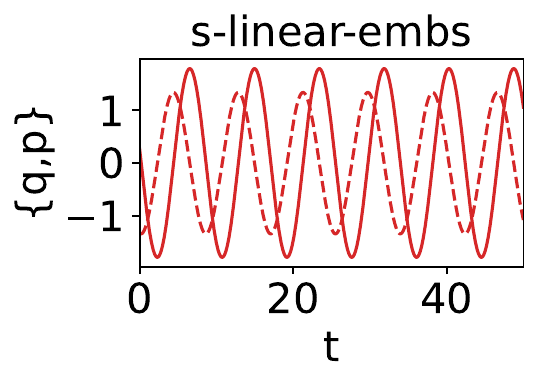}	
		\includegraphics[width=0.24\linewidth]{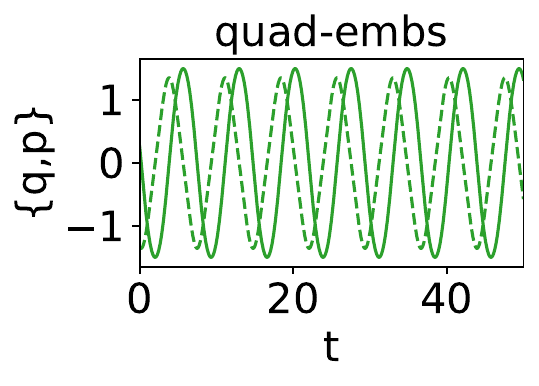}
		\includegraphics[width=0.24\linewidth]{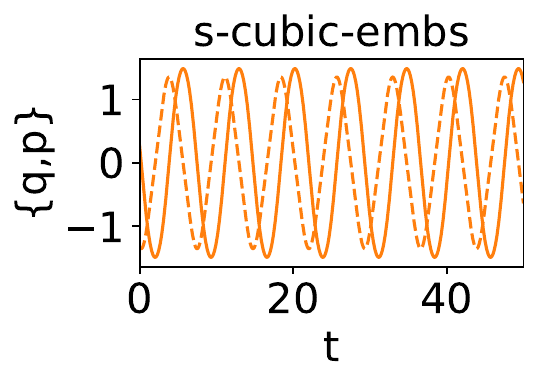}
		\caption{The best scenario for \cubicembs.}
	\end{subfigure}
	\begin{subfigure}[b]{1\textwidth}
		\includegraphics[width=0.24\linewidth]{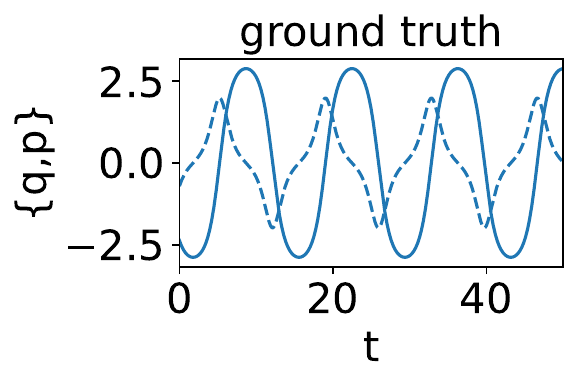}				\includegraphics[width=0.24\linewidth]{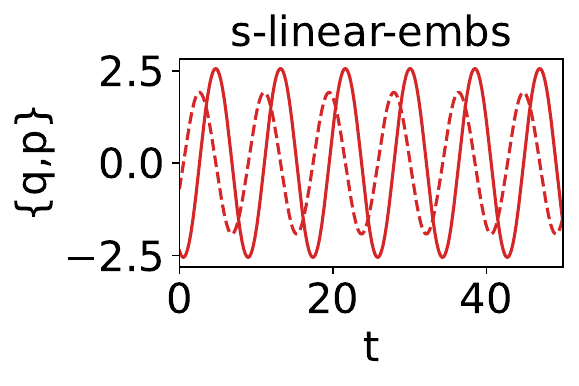}	
		\includegraphics[width=0.24\linewidth]{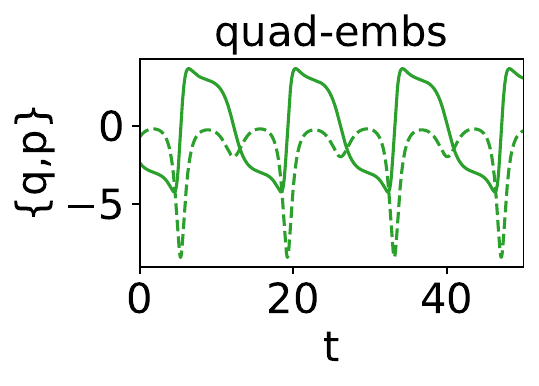}
		\includegraphics[width=0.24\linewidth]{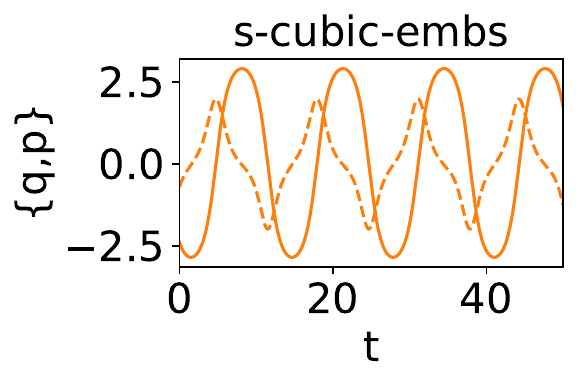}
		\caption{The worst scenario for \cubicembs.}
	\end{subfigure}
	\caption{Nonlinear pendulum: Comparisons of the time-domain simulations using learned different representations test conditions for which \cubicembs\ performs the best and worst among the considered $25$ test conditions.}
	\label{fig:pendulum_timesimulation}
\end{figure}

\subsection{Nonlinear oscillator} 
The second example involves a harmonic nonlinear oscillator, described by the Hamiltonian as follows:
\begin{equation}
	\Hamiltonian(q,p) = \frac{p^2}{2}+\frac{q^2}{2}+\frac{q^2}{4},
\end{equation}
where $q$ and $p$ denote the position and momentum of the oscillator, respectively. To gather training data, we select $50$ data points in the time interval $[0,4]$ for a given initial condition and consider  $20$  random initial conditions. We consider  $\{q,p\} \in [-2,2]\times [-2,2]$. Similar to the previous example, we constrain the initial conditions so that $\cH(q,p) \leq 1$.

Next, we learn suitable embeddings using \linearembs, \quadembs, and \cubicembs\ by setting the dimension of it to two. Once we have the desired embeddings, we evaluate their performance using $25$ test initial conditions. Moreover, for testing, we consider a significantly longer time than the training time, which is $[0,50]$, and take $5000$ points in the testing time interval.
First, in \Cref{fig:nlo_phase}, we present the phase space plots using the learned $\{q,p\}$ using different methods. We observe that \linearembs\ could not accurately capture the phase space, while \quadembs\ and \cubicembs\ capture it well, at least in the eyeball norm. 

\begin{figure}[tb]
	\includegraphics[width=0.90\linewidth]{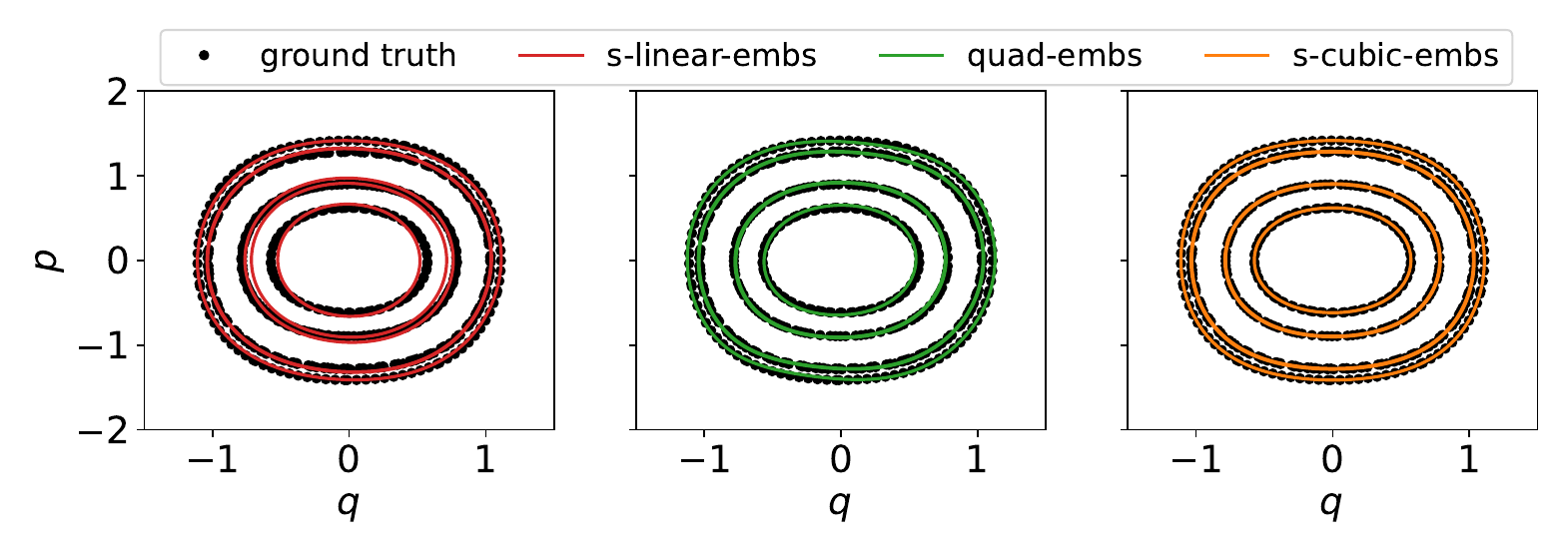}
	\caption{Nonlinear oscillator: The figure shows a comparison of various learned models with the ground truth in the phase space for test initial conditions.}
	\label{fig:nlo_phase}
\end{figure}

To conduct a qualitative comparison, we utilize the measure given in \eqref{eq:error_measure}.  Based on this, we calculate the error for each test initial condition and show them in a violin plot in \Cref{fig:nlo_errorplot}. It shows the superior performance of \cubicembs\ as compared to the other two.  It is worth noting that  \quadembs\ and \cubicembs\ yield solutions that appear similar in phase space; however, \quadembs\ exhibits a significant shift in the solutions compared to \cubicembs, leading to a higher error in the measure \eqref{eq:error_measure} for \quadembs.

\begin{figure}[tb]
	\includegraphics[width=0.350\linewidth]{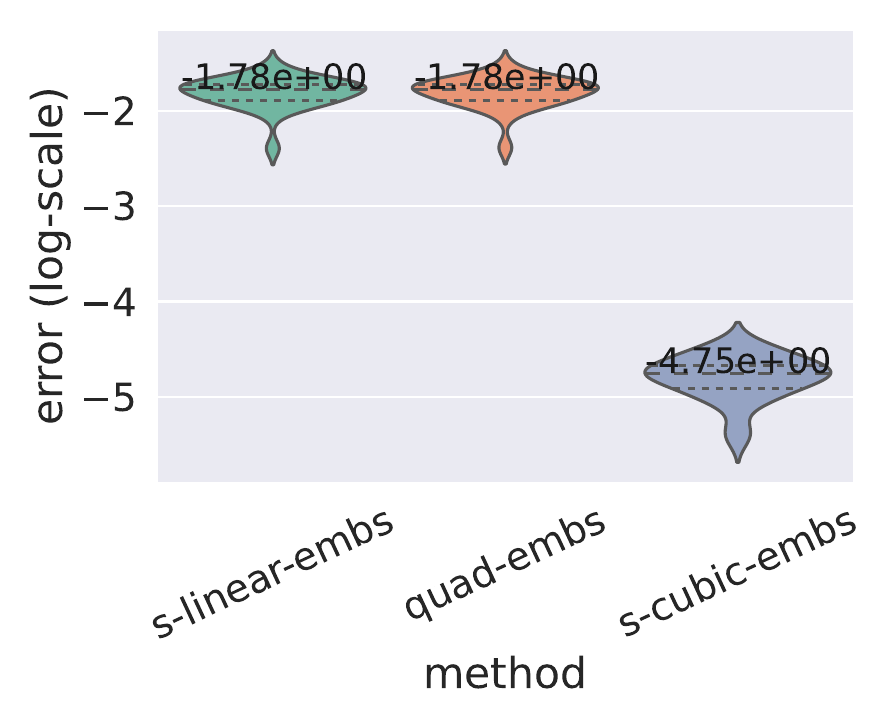}
	\caption{Nonlinear oscillator: A qualitative performance analysis of different methods for test initial conditions. The middle dashed line shows the median, with the precise number, across all test initial conditions.} 
	\label{fig:nlo_errorplot}
\end{figure}

Furthermore, in \Cref{fig:nlo_time}, we present the time-domain simulations of two test initial conditions, where \cubicembs\ performs the worst and best among the considered test conditions as compared to the ground truth. We observe that \linearembs\ and \quadembs\ yield solutions that appear similar to the ground truth but have a noticeable shift in the solution as compared to the ground truth. Conversely, \cubicembs\ outperforms both \linearembs\ and \quadembs\ in terms of the amplitude and period, as can be clearly seen in the solutions at around time $t=50$.

\begin{figure}[tb]
	\centering
	\begin{subfigure}[b]{1\textwidth}
		\includegraphics[width=0.24\linewidth]{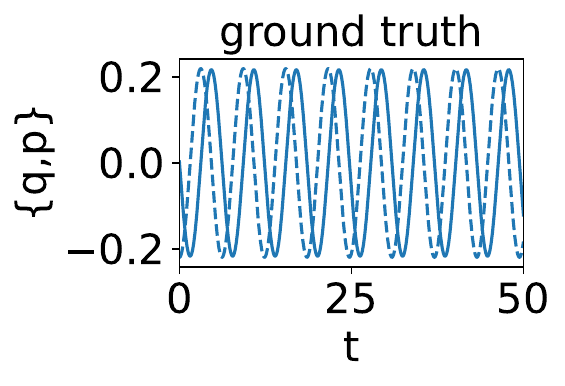}			\includegraphics[width=0.24\linewidth]{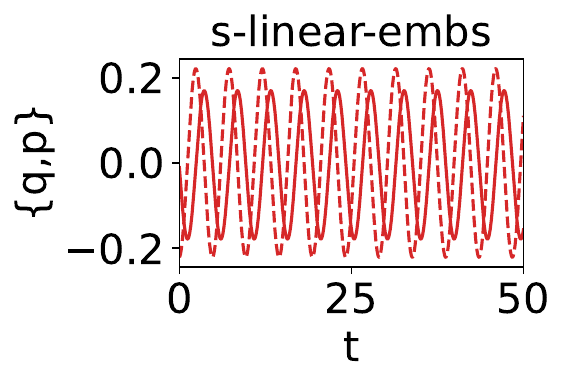}	
		\includegraphics[width=0.24\linewidth]{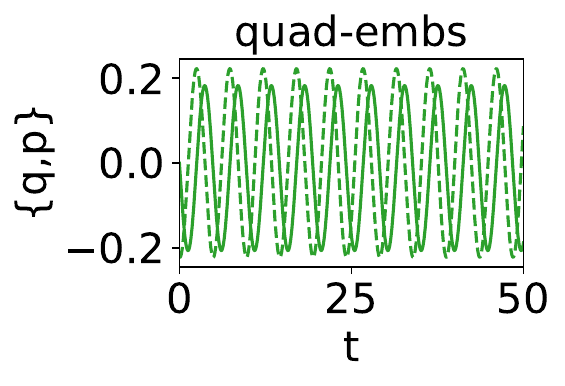}
		\includegraphics[width=0.24\linewidth]{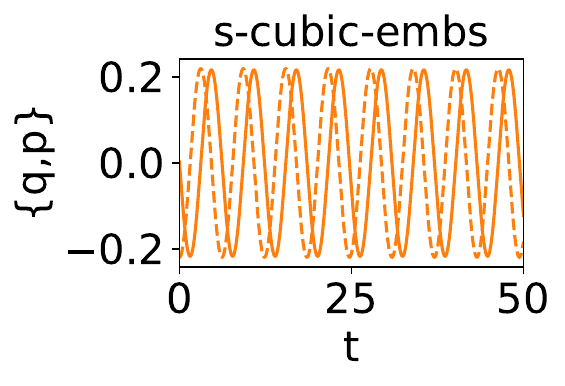}
		\caption{The best scenario for \cubicembs.}
	\end{subfigure}
	\begin{subfigure}[b]{1\textwidth}
		\includegraphics[width=0.24\linewidth]{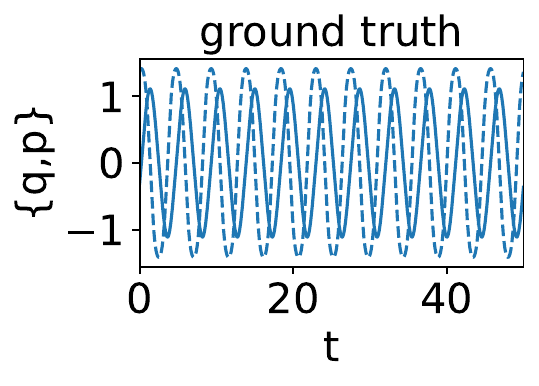}			\includegraphics[width=0.24\linewidth]{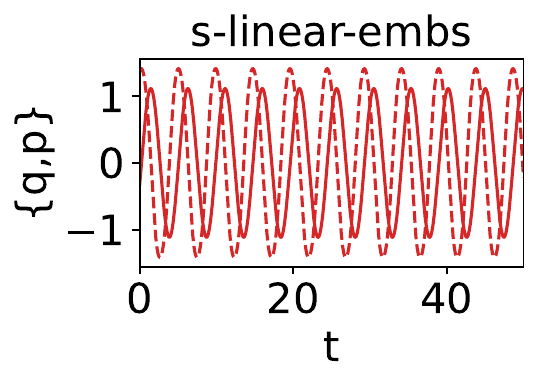}	
		\includegraphics[width=0.24\linewidth]{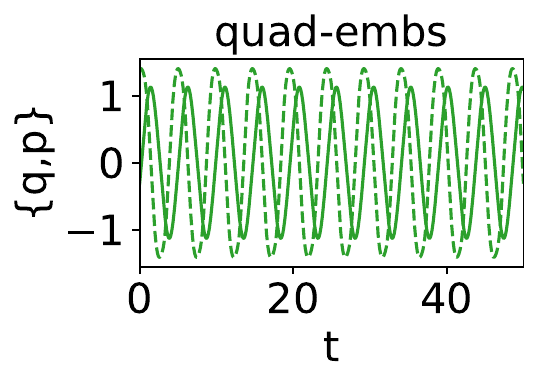}
		\includegraphics[width=0.24\linewidth]{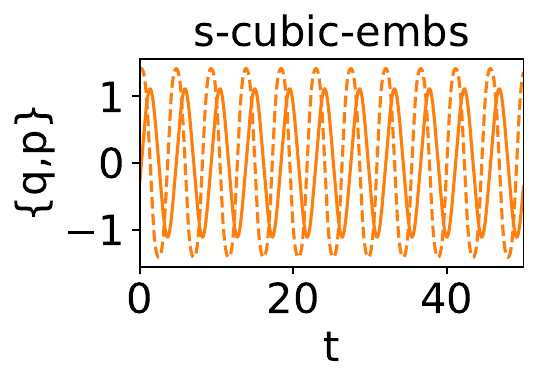}
		\caption{The worst scenario for \cubicembs.}
	\end{subfigure}
	\caption{Nonlinear oscillator: Comparisons of the time-domain simulations using learned different representations test conditions for which \cubicembs\ performs the best and worst among the considered $25$ test conditions with respect to the ground truth.}
	\label{fig:nlo_time}
\end{figure}
\subsection{Lotka–Volterra model} 
In the last low-dimensional example, we consider the Lotka-Volterra model, which is a well-known mathematical model used to describe predator-prey populations' dynamics. It is a Hamiltonian system with the Hamiltonian as follows:
\begin{equation*}
	\Hamiltonian(q,p) = p-e^{p}+2q-e^q.
\end{equation*}

\begin{figure}[tb]
	\includegraphics[width=0.9\linewidth]{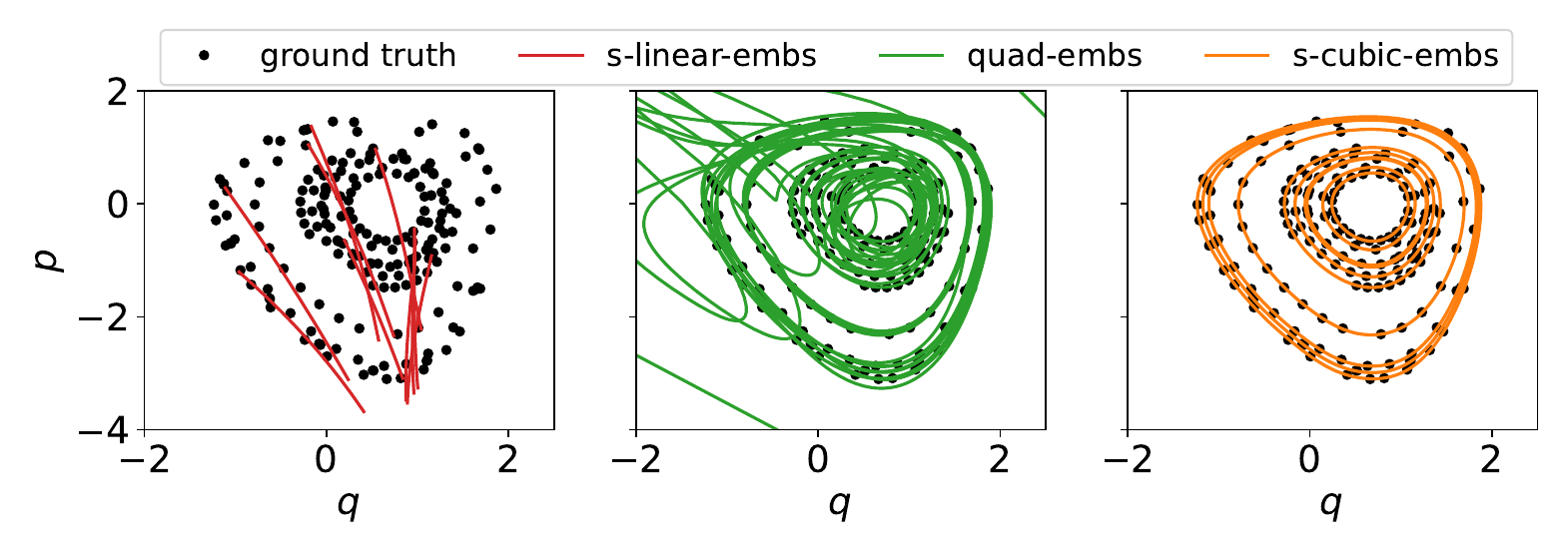}
	\caption{Lotka–Volterra model: The figure shows a comparison of various learned models with the ground truth in the phase space for test initial conditions.}
	\label{fig:lv_phaseplot}
\end{figure}

To learn the dynamics, we begin by collecting training data using $20$ random initial conditions in the range of $[-1.5,1.5]\times [-1.5,1.5]$ for $q$ and $p$. From the set of initial conditions, we choose the initial conditions for which $\cH(q,p) \leq 4$. For each initial condition, we consider $100$ points within the time interval of $[0,10]$. With the collected data, we proceed to learn the desired embeddings in a four-dimensional space using all three methods.

To assess the quality of these methods, we conducted tests using $25$ different initial conditions. For the testing phase, we consider the time interval $[0,50]$ and take $10000$ points within the testing time interval. Note that the testing time interval is five times larger than the training time interval. Like previous examples, we first plot the phase space in \Cref{fig:lv_phaseplot},  revealing that \linearembs\ does not capture the dynamics at all. This limitation could be attributed to the continuous spectra of the Lotka-Volterra model, making it challenging to derive a universal linearized model using finite-dimensional embeddings.
Moving forward, we evaluated \quadembs, which partially captured the dynamics, but it showed signs of instability when integrated over a longer time period. In contrast, \cubicembs\ demonstrated remarkable accuracy in learning the dynamics.

Additionally, in \Cref{fig:lv_time}, we plot the time-domain simulations using the test initial conditions for which \cubicembs\ performs the best and the worst with respect to the ground truth. The figures clearly demonstrate that \cubicembs\ is the only method among the three considered methods capable of accurately reproducing the dynamics for all test cases.

\begin{figure}[tb]
	\centering
	\begin{subfigure}[b]{1\textwidth}
		\includegraphics[width=0.24\linewidth]{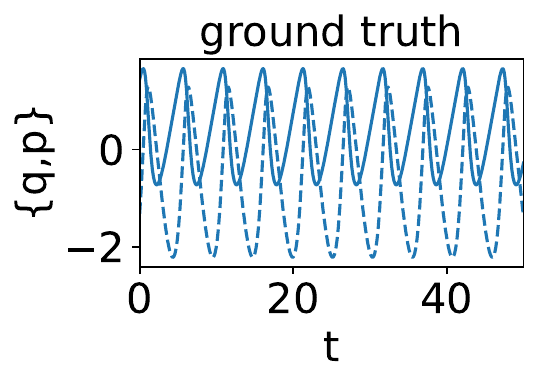}				\includegraphics[width=0.24\linewidth]{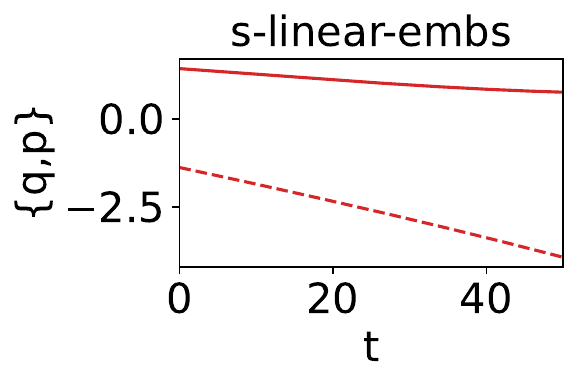}	
		\includegraphics[width=0.24\linewidth]{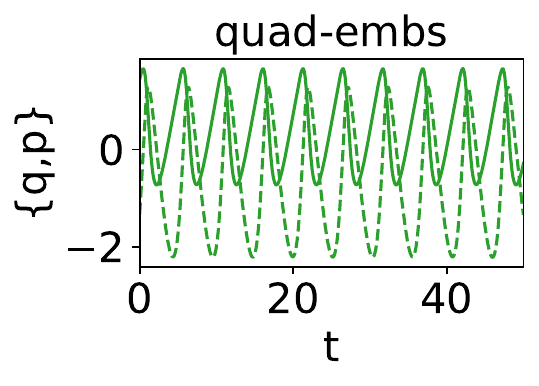}
		\includegraphics[width=0.24\linewidth]{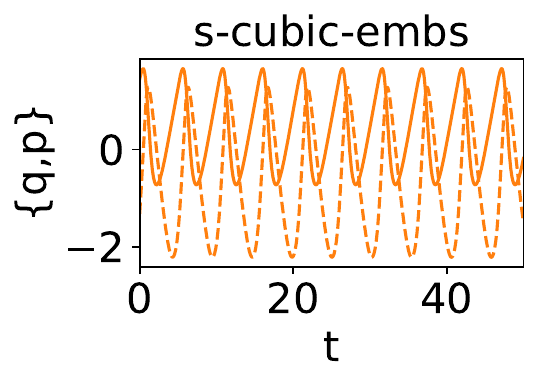}
		\caption{The best scenario for \cubicembs.}
	\end{subfigure}
	\begin{subfigure}[b]{1\textwidth}
		\includegraphics[width=0.24\linewidth]{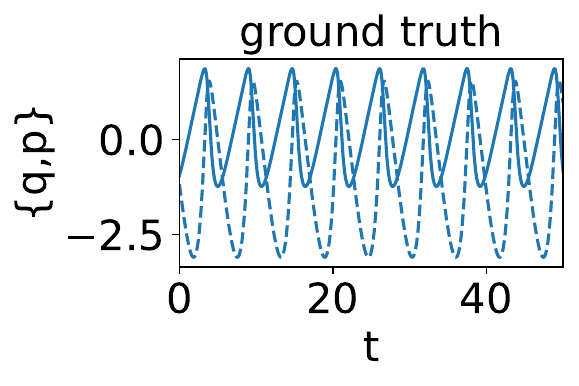}				\includegraphics[width=0.24\linewidth]{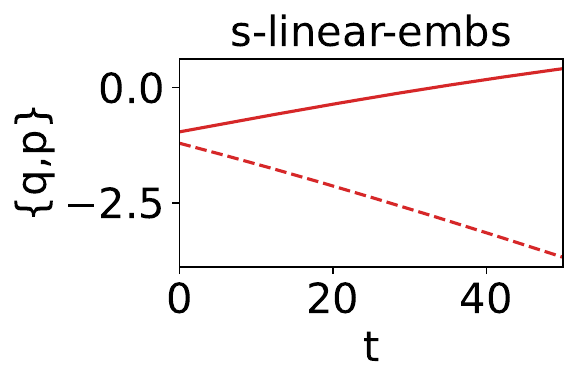}	
		\includegraphics[width=0.24\linewidth]{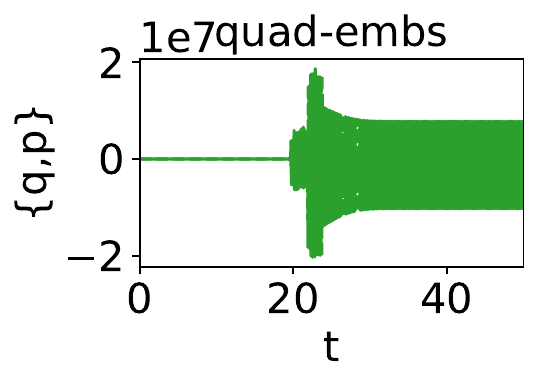}
		\includegraphics[width=0.24\linewidth]{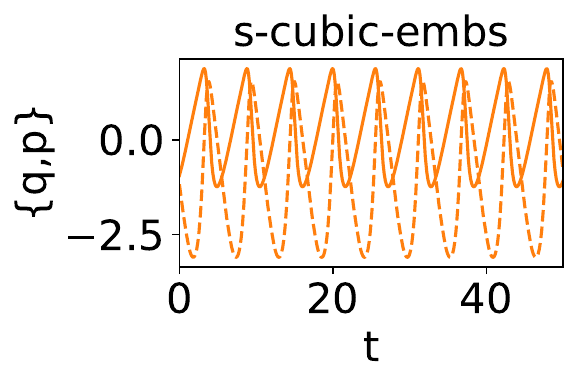}
		\caption{The worst scenario for \cubicembs.}
	\end{subfigure}
	\caption{Lotka–Volterra model: Comparisons of the time-domain simulations using the learned different embeddings test conditions for which \cubicembs\ performs the best and worst among the considered $25$ test conditions with respect to the ground truth.}
	\label{fig:lv_time}
\end{figure}

\section{An Extension to High-Dimensional Data}\label{sec:high_dimensional_case}
The dynamics of various physical processes are governed by partial differential equations. As a result, the data collected for such processes are often high-dimensional. However, these data often exhibit significant spatial and temporal correlations, allowing them to represent using a low-dimensional manifold effectively. To obtain such a low-dimensional representation, we utilize the dominant POD basis obtained through an SVD of the high-dimensional training data. Furthermore, since we deal with symplectic data and assume that the underlying model is also symplectic, we incorporate the concept of co-tangent lifting \cite{peng2016symplectic} to determine the dominant basis using the high-dimensional snapshots, as it preserves the symplectic structure.  Both the position and momenta data are concatenated in the snapshot matrix. For this, let us denote the dominant basis by $\bV$. It is then followed by determining the projection matrix as $\cV = {\small \begin{bmatrix} \bV & 0 \\ 0 & \bV \end{bmatrix}}$.
Using the projection matrix, we derive a low-dimensional representation, which we refer to it as POD coordinates. The time-derivative information for these coordinates is estimated using a five-points stencil method. These coordinates are then employed for learning suitable embeddings using \linearembs, \quadembs, and \cubicembs\ to describe the dynamics of the POD coordinates. In the following, we present two examples with high-dimensional data and discuss the efficiency of the proposed methodologies.

\subsection{Nonlinear Schrödinger equation}
Towards learning dynamics using high-dimensional data, we begin by considering the nonlinear Schrödinger (NLS) equation, which has found extensive applications in analyzing the waves on the water surface and the propagation of the light in fibers. We consider  the cubic Schrödinger equation given by \cite{karasozen2018energy}:
\begin{equation}\label{eqn:schrod}
\begin{aligned}
& i \dfrac{\partial u}{\partial t} + \dfrac{1}{2} u_{\zeta\zeta} +  |u|^2 u=0, \\
& u(0,\zeta) = u_0(\zeta), & \zeta \ \in \; \Omega := [-10,10],
\end{aligned}
\end{equation}
subject to periodic boundary conditions, and $u_{\zeta\zeta}$ denotes the double-derivative of $u$ with respect to the space $\zeta$. To obtain the canonical Hamiltonian form of the NLS equation \eqref{eqn:schrod}, we express the complex-valued solution $u$ in terms of its imaginary and real parts, i.e., $u=q+ \imath p$. The Hamiltonian of the NLS equation is then given by
\begin{equation*}
\Hamiltonian(u) = \dfrac{1}{2}\int_{\Omega} \bigg[
\alpha\bigg(\dfrac{\partial q}{\partial \zeta}\bigg)^2+
\alpha\bigg(\dfrac{\partial p}{\partial \zeta}\bigg)^2-
\dfrac{\beta}{2} (q^2+p^2)^2
\bigg]\,d\zeta.
\end{equation*}

Following \cite{yildizetal23}, we use a three-point central difference approximation for the $\partial_{\zeta\zeta}$ and consider  256 points in the designated spatial domain. We collect the data using the initial condition $u_0(\zeta)=\text{sech}(\zeta/2)$ over a time-interval  $[0,160]$, with a total of $3200$ points within the time-interval. For our analysis, we divide the data into two halves: the first half is utilized for training purposes, while the second half is reserved for testing the performance of the different methods.

Since the data are high-dimensional ($N = 2\cdot 256$), our first step is to identify the POD coordinates. For this, we employ the SVD of the training data and observe its decay in \Cref{fig:NLS_svd}. We notice a reasonable decay, thus indicating a good low-dimensional representation using a linear projection. By considering the first two POD bases, we capture more than $94\%$ of the energy present in the training data. As a result, we have four-dimensional POD coordinates by utilizing the co-tangent projection of the high-dimensional data. 
\begin{figure}[tb]
	\centering
	\begin{subfigure}[b]{0.35\textwidth}
		\includegraphics[width=1\linewidth]{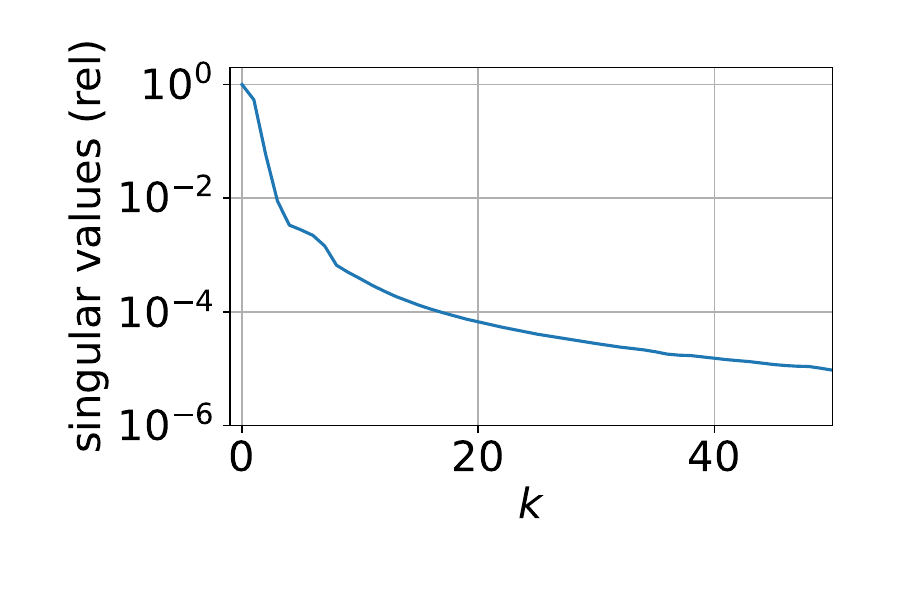}
	\end{subfigure}
	\begin{subfigure}[b]{0.35\textwidth}
		\includegraphics[width=1\linewidth]{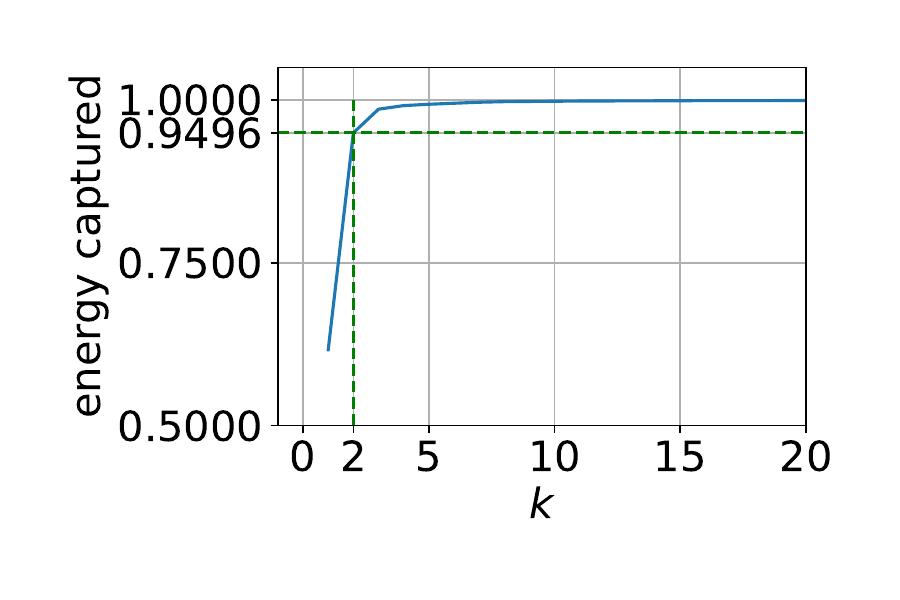}
	\end{subfigure}
	\caption{NLS example: The figure shows the delay of the singular values, obtained using the training data. The right plot indicates the energy captured using $2$ dominant modes.}
	\label{fig:NLS_svd}
\end{figure}

Having obtained the POD coordinates, our next objective is to learn appropriate embeddings of dimension four using three different methodologies. In \Cref{fig:nls_pod_coeffs}, we depict the reconstruction of the POD coordinates through these identified embeddings using different methods. We observe that \quadembs\ fails to capture the dynamics in both training and testing phases, whereas \linearembs\ and \cubicembs\ perform well in capturing the high-level dynamic features. This is more apparent when the learned solutions are plotted in the phase-space in the testing phase; see \Cref{fig:nls_pod_coeffs_phasesapce} (first four columns).

\begin{figure}[tb]
	\centering
	\begin{subfigure}[t]{0.24\textwidth}
		\includegraphics[width=1\linewidth]{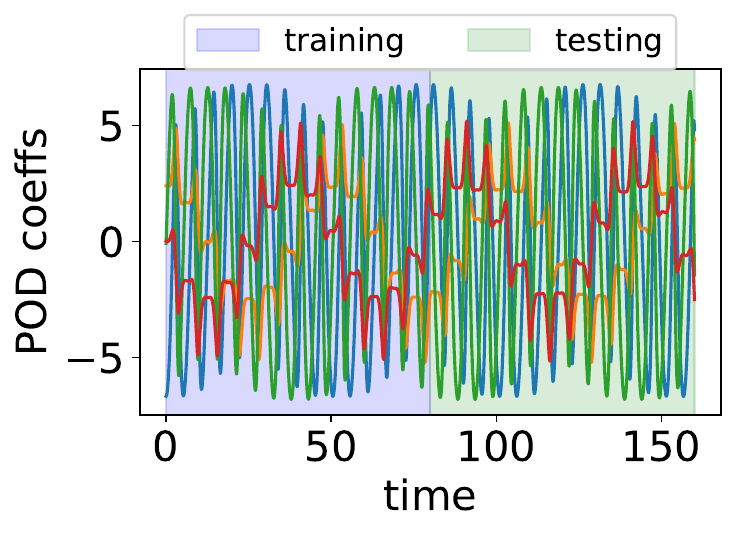}			
		\caption{ground truth}
	\end{subfigure}
	\begin{subfigure}[t]{0.24\textwidth}
		\includegraphics[width=1\linewidth]{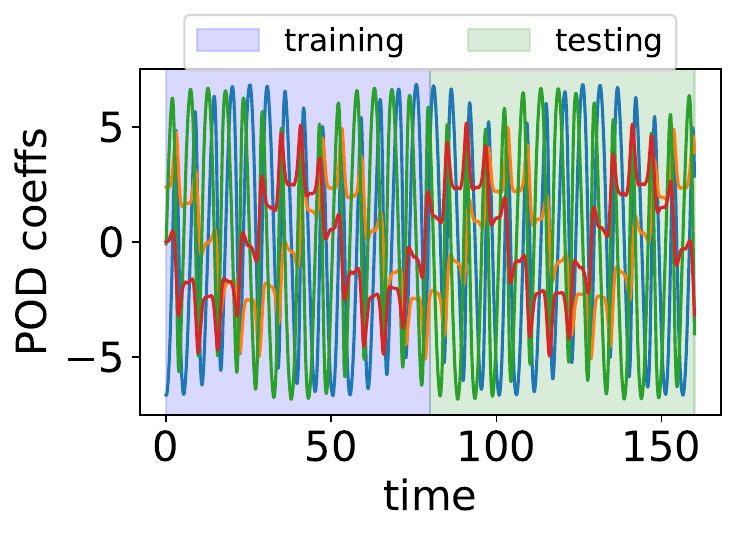}
		\caption{using \linearembs}
	\end{subfigure}
	\begin{subfigure}[t]{0.24\textwidth}
		\includegraphics[width=1\linewidth]{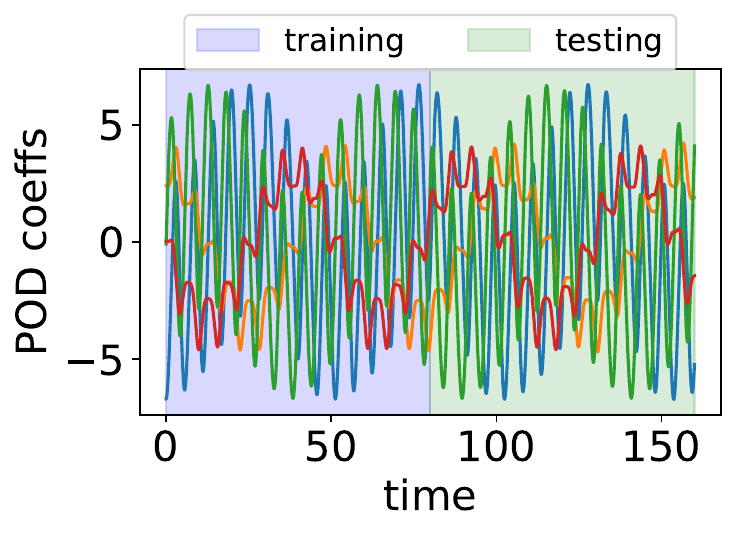}
		\caption{using \quadembs}
	\end{subfigure}
	\begin{subfigure}[t]{0.24\textwidth}
		\includegraphics[width=1\linewidth]{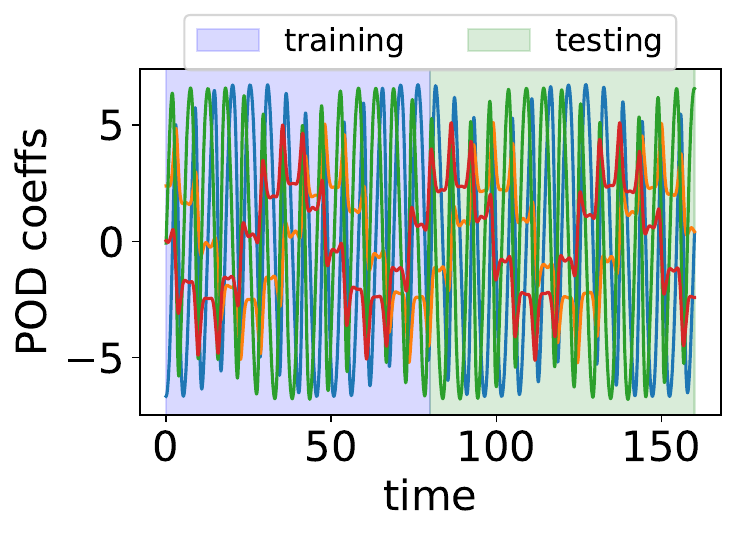}
		\caption{using \cubicembs}
	\end{subfigure}
	\caption{NLS model: A comparison of the transient responses of the learned models using different methods with the ground truth model.}
	\label{fig:nls_pod_coeffs}
\end{figure}

\begin{figure}[tb]
	\centering
	\begin{subfigure}[t]{1\textwidth}
		\includegraphics[width=0.19\linewidth]{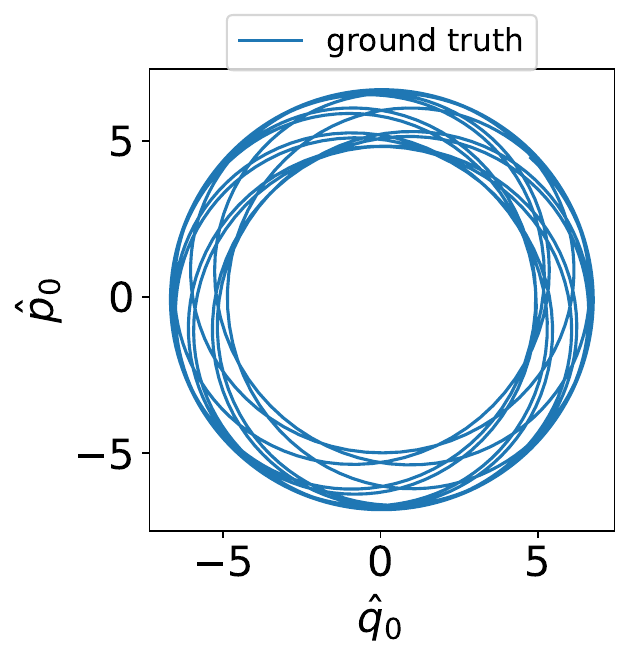}
		\includegraphics[width=0.19\linewidth]{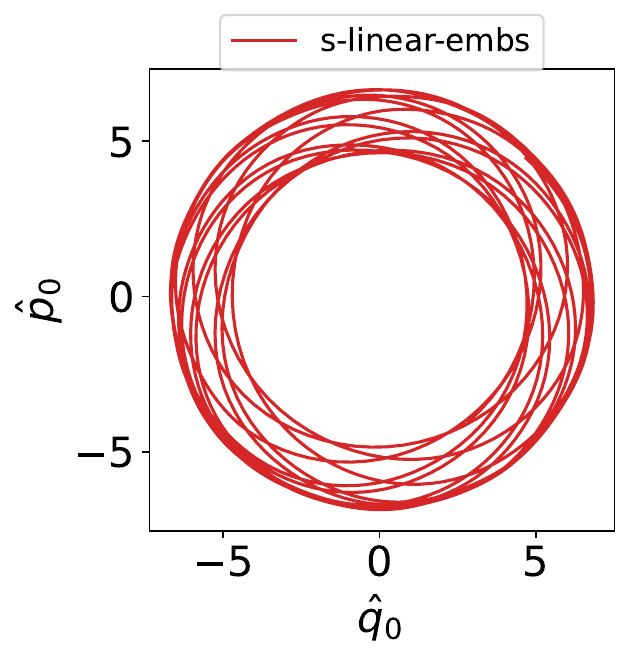}
		\includegraphics[width=0.19\linewidth]{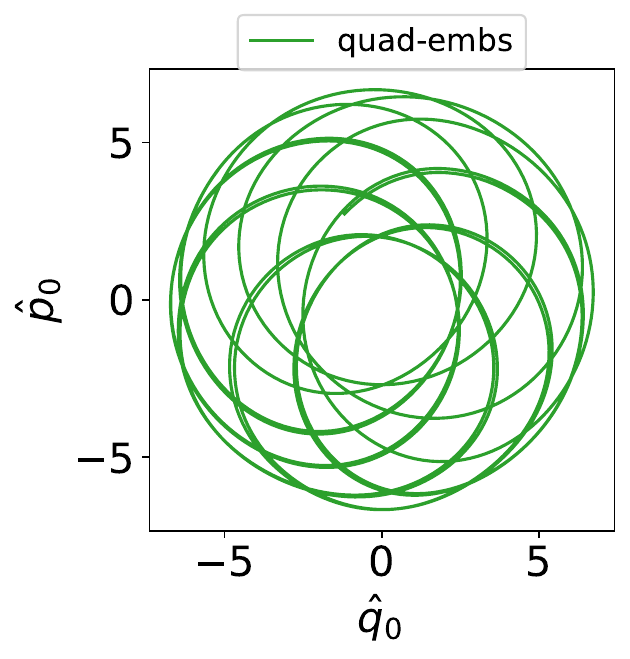}
		\includegraphics[width=0.19\linewidth]{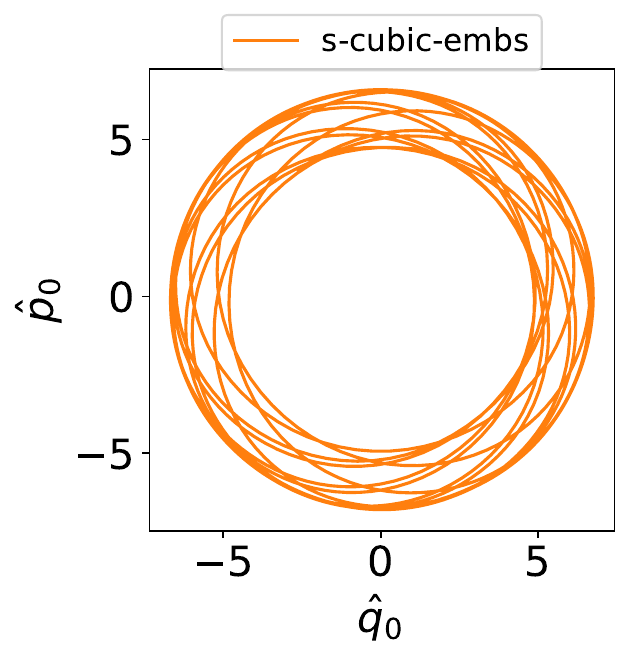}
		\includegraphics[width=0.19\linewidth]{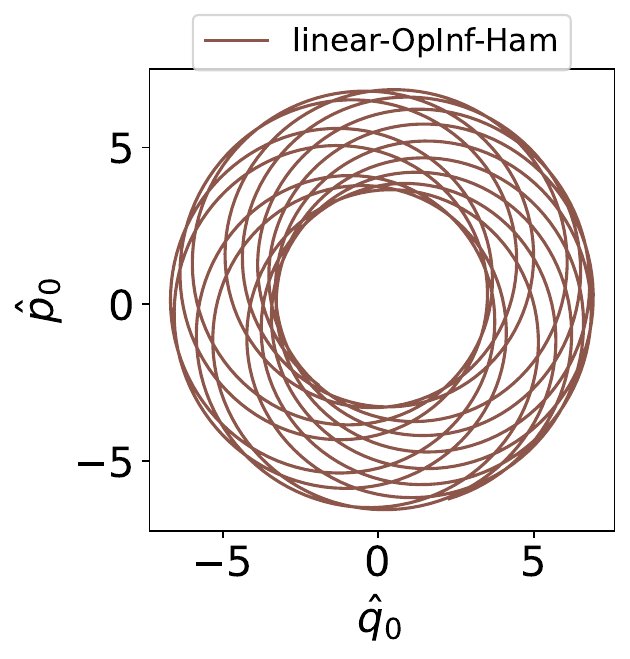}
		\caption{Phase space for $(\hat q_0, \hat p_0)$.}
	\end{subfigure}
	\begin{subfigure}[t]{1\textwidth}
		\includegraphics[width=0.19\linewidth]{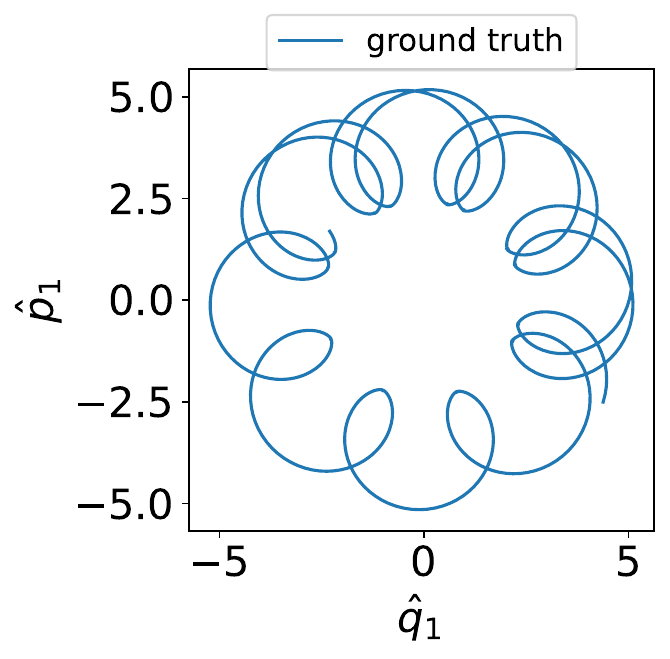}
		\includegraphics[width=0.19\linewidth]{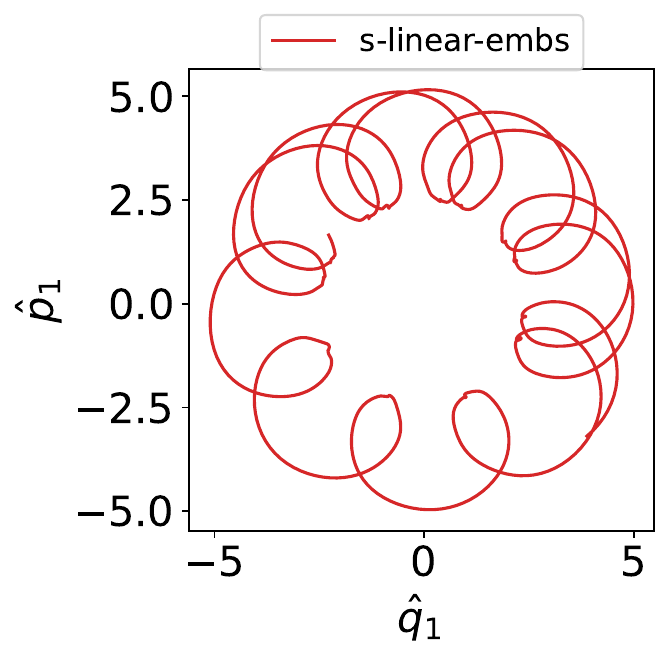}
		\includegraphics[width=0.19\linewidth]{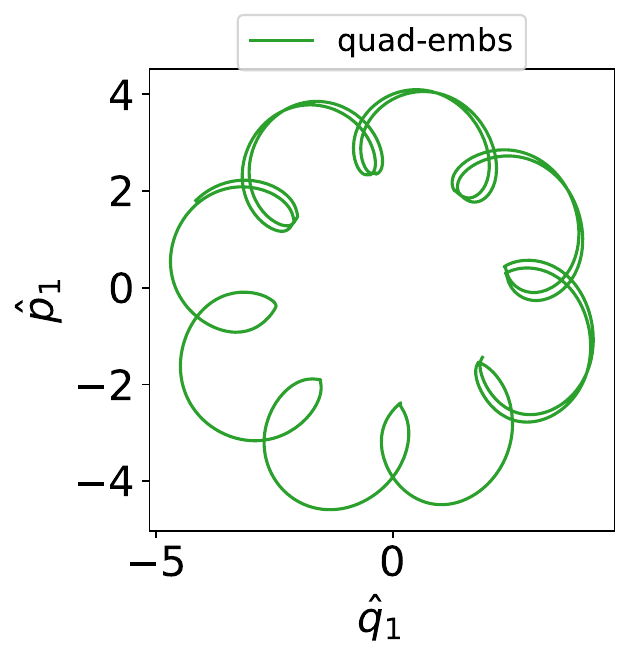}
		\includegraphics[width=0.19\linewidth]{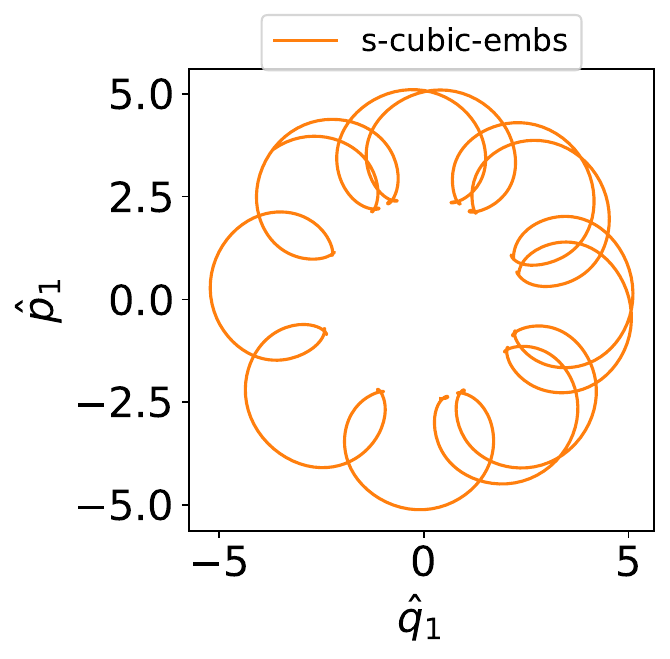}
		\includegraphics[width=0.19\linewidth]{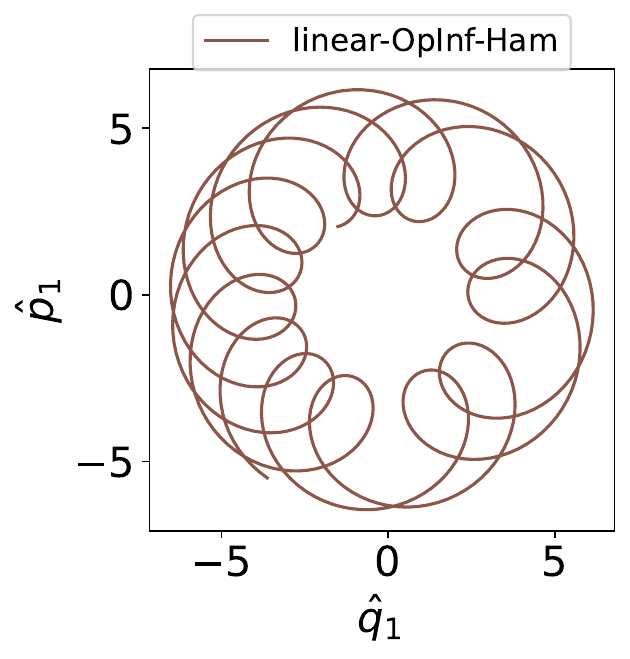}
		\caption{Phase space for $(\hat q_1, \hat p_1)$.}
	\end{subfigure}
	\caption{NLS model: A comparison of different methods in the testing phase. We show the positions and momenta, which are part of the POD coordinates in phase space.}
	\label{fig:nls_pod_coeffs_phasesapce}
\end{figure}

For a qualitative comparison, in \Cref{tab:nls_model_pod_err}, we also report the relative $L_2$-norm of the error for the POD coefficients for the training and testing phases. We notice that \linearembs\ yields the best results among the considered cases. Intuitively, one might expect \cubicembs\ to be as good as \linearembs\ since setting blocks of $\bQ$ in \eqref{eq:hamiltonain_quartic} to zero boils down to \linearembs\ case. However, when the dynamics are obtained using a numerical integration method, an over-fitting is observed for \cubicembs, thus leading to poorer performance than \linearembs. But the performance of \cubicembs\ can be improved by fine-tuning hyper-parameter, e.g., with respect to the $L_2$-regularization of the parameters for \cubicembs. Finding good hyper-parameters is our potential avenue for our future research.

\begin{table}[tb]
	\begin{tabular}{|c|c|c|c|} \hline
		& \multicolumn{3}{c|}{Method}                                                                                \\ \hline
		& \linearembs & \quadembs & \cubicembs  \\ \hline
		Training phase & {$\mathbf{4.37\cdot 10^{-3}}$}                          & $3.37\cdot 10^{-1}$                         & $2.77\cdot 10^{-1}$                                             \\ \hline
		Testing phase  & $\mathbf{2.69\cdot 10^{-2}}$                         & $1.60\cdot 10^{0}~~$                        & $1.34\cdot 10^{0}~~$                         \\   \hline                 
	\end{tabular} 
	\caption{NLS model: The table shows the relative $ L_2$ errors of the POD coordinates in the training and testing phases obtained using various methods.}
	\label{tab:nls_model_pod_err}
\end{table}

Furthermore, we conducted a  performance comparison between \linearembs\ and the operator inference (\opinf) approach \cite{sharma2022hamiltonian,gruber2023canonical} that also preserves the canonical Hamiltonian structure. Using \opinf, we also learn a linear system, aiming to describe the dynamics of the POD coordinates. In \Cref{fig:nls_pod_coeffs_phasesapce}, we show the results for \opinf\ (the fifth column) in the testing, clearly demonstrating that \opinf\ fails to capture the dynamics. These findings underscore the significance of learning the correct embeddings for dynamic system modeling, which we achieve, for example,  for \linearembs\ by unleashing the power of deep learning.

Lastly, we aim at reconstructing the solution over the considered entire spatial domain using the POD coordinates. We explore three different approaches for this purpose. 
The first and simplest approach involves using the POD basis to reconstruct the solution, where the solution is reconstructed through a linear combination of the POD coordinates. We refer to it as \emph{linear-decoder}. Next, we consider the quadratic-manifold approach \cite{geelen2023operator}, which constructs the solution on the entire spatial domain using a quadratic function of the POD coordinates. As a result, we refer to it as \emph{quad-decoder}. Finally, we explore a convolution neural network (CNN) based approach to reconstruct the solution using the POD coordinates. We refer to it as \emph{convo-decoder}.  The architecture of the CNN is discussed in \Cref{appendix:decoder}, and therein, the training of both the quadratic and CNN-based decoders is also provided. 

Using these decoder approaches, we plot the reconstructions of the solution of the full spatial domain in \Cref{fig:nls_decoder}. Note that we have used the learned embeddings, obtained using \linearembs\ as it performs the best among the considered ones (see \Cref{tab:nls_model_pod_err}).
Upon evaluation, we notice that all three decoder approaches demonstrate reasonable performance. Particularly, we emphasize that the linear-decoder approach also performs well since more than $94\%$ of the energy has been captured by the selected POD modes.
Moreover, we notice that the quadratic decoder works slightly worse than the linear decoder in terms of solution reconstruction in the testing phase. It could be due to the over-fitting of the quad-decoder despite utilizing the $L_2$ regularizer. On the other hand, the convo-decoder emerges as the most effective method among those considered (notice the limits of the bar-plots in \Cref{fig:nls_decoder} and \Cref{tab:nls_model_decoder}). 

\begin{figure}[tb]
	\centering
	\includegraphics[width=.99\linewidth, trim = 0cm 0cm 0cm 0cm, clip]{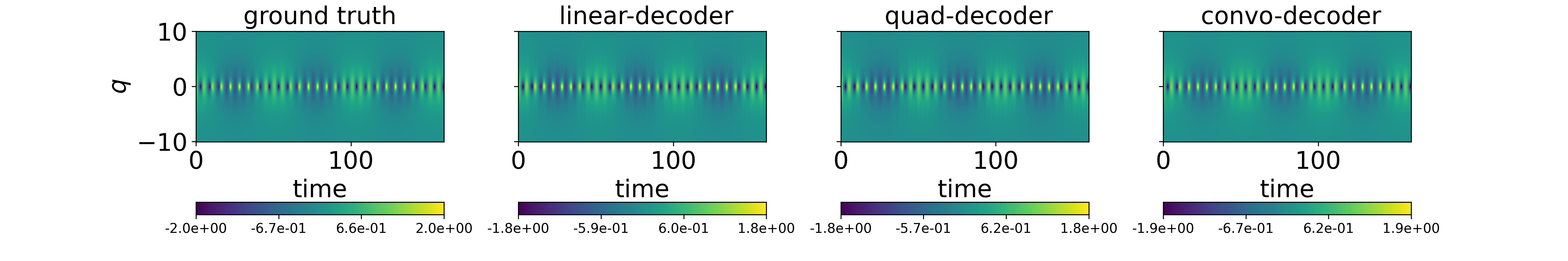}			
	\includegraphics[width=0.99\linewidth, trim = 0cm 0cm 0cm 0cm, clip]{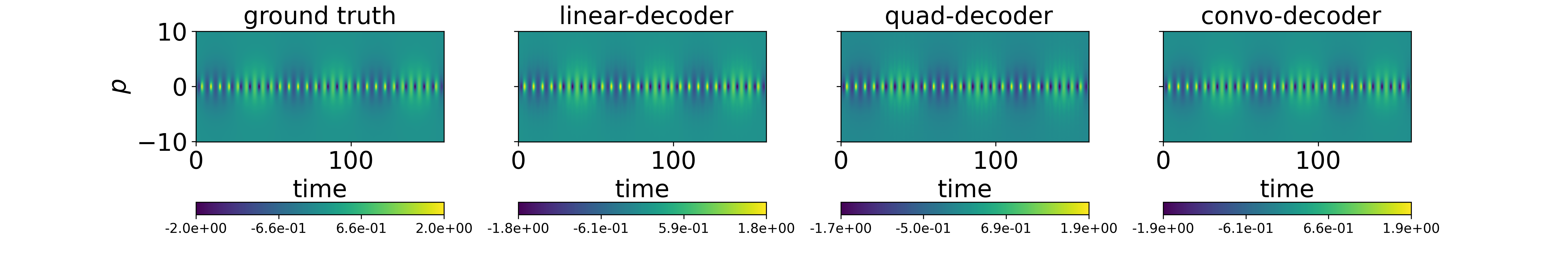}
	\caption{NLS example: A comparison of the reconstructions of $\bq$ and $\bp$ using the learned POD coordinates $\hat\bq$ and $\hat\bp$ using \linearembs\ by means of various decoder methods.}
	\label{fig:nls_decoder}
\end{figure}

\begin{table}[t]
	\begin{tabular}{|c|c|c|c|} \hline
		& \multicolumn{3}{c|}{Methods}                                                                                \\ \hline
		& linear-decoder & quad-decoder & convo-decoder  \\ \hline
		Training phase & {$ 7.10\cdot 10^{-4}$}                          &  $7.40\cdot  10^{-4}$                        & $\mathbf{4.62\cdot 10^{-4}}$                                             \\ \hline
		Testing phase  & {$ 3.00\cdot 10^{-3} $}                         & $3.12\cdot 10^{-3}$                        & $\mathbf{2.78\cdot 10^{-3}}$                          \\   \hline                 
	\end{tabular} 
	\caption{NLS model: The table shows the mean $L_2$-errors of the reconstructed and ground truth solutions using various decoder approaches.}
	\label{tab:nls_model_decoder}
\end{table}
\subsection{Wave example}
In our last example, we consider a linear wave equation of the form:
\begin{equation}\label{eqn:wave}
\begin{aligned}
&u_{tt} =u_{\zeta\zeta},\\
& u(t_0,\zeta) = u_0(\zeta), & \zeta \in \Omega:=[-10,10],  
\end{aligned}
\end{equation}
with a periodic boundary condition. The wave equation is an example of a Hamiltonian PDE \cite{bridges2006}, and its corresponding Hamiltonian is defined as 
\begin{equation*}
\Hamiltonian(u) = \dfrac{1}{2}\int_{\Omega}
\left(q_\zeta^2+p^2\right)\,d\zeta,
\end{equation*}
where $p=u_t$ and $q=u$, and $q_\zeta$ denotes the partial derivative of $q$ with respect to $\zeta$. Discretizing the PDE as discussed in \cite{yildizetal23}, we can arrive at a discretized model as follows:
\begin{equation}\label{eqn:Wave-Ham-ODE}
\frac{d \bl z}{d t}=\bl K \bl z, \quad \text{with} ~~\bl z=\begin{bmatrix}
\bl q\\ \bl p
\end{bmatrix} ~~\text{and}~~ \bl{K}=\begin{bmatrix}
\mathbf{0}& \bI_N\\ \bD_{\zeta\zeta}& \mathbf{0}
\end{bmatrix}.
\end{equation}
In the above equation, $ \bD_{\zeta\zeta}\in \mathbb{R}^{N\times N}$ denotes an approximation of $\partial_{\zeta\zeta} $, $\mathbf{0}\in \mathbb{R}^{N\times N}$ and $\bI_N\in \mathbb{R}^{N\times N}$ are the zero and identity matrices, and $(\bl q, \bl p)$ are the discretized $(q,p)$, and $N$ denotes the number of grid points in the space.

Next, we generate training data using a parameterized initial condition $u_0(\zeta;\mu)$, where $u_0(\zeta;\mu)=\text{sech}(\mu \cdot \zeta)$ with $\mu$ ranging from $0.5$ to $1.4$. We collect data using $10$ different values of $\mu$, i.e., $\mu = \{0.5, 0.6, \ldots, 1.3, 1.4\}$. Among these, we designate three values of $\mu$ ($\mu = \{0.7, 1.0, 1.2\}$) for testing, while the remaining seven values serve as training data.  Moreover, we consider $256$ equidistant points in the spatial domain $\Omega = [-10,10]$, resulting in a model of dimension $512$. For each training initial condition, we consider $501$ points in the time interval $[0,25]$. 

As considered in the previous example, we first aim to project the high-dimensional data onto a low-dimensional subspace and obtain the POD coordinates. With that goal, in \Cref{fig:wave_svd}, we examine the decay of the singular values, indicating a relatively slow decay of it. It is expected as the considered system has a transport phenomenon, which often has a slow-decay of Kolmogorov $N$-width \cite{kolmogoroff1936uber, greif2019decay}. Next, we consider the three dominant modes to determine the projection matrix, which captures only $52.7\%$ of the energy present in the training data. We obtain the POD coordinates using the principle of co-tangent lifting for canonical Hamiltonian systems, encompassing the reduced momenta and position, respectively, denoted by $\hat\bp \in \R^3$ and $\hat\bq \in \R^3$. 

\begin{figure}[tb]
	\centering
	\begin{subfigure}[b]{0.35\textwidth}
		\includegraphics[width=1\linewidth]{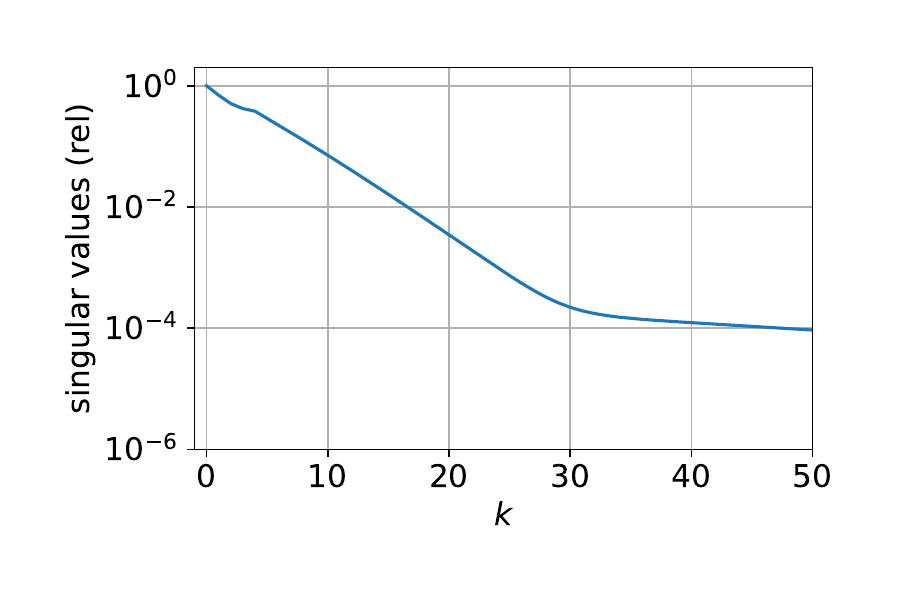}
	\end{subfigure}
	\begin{subfigure}[b]{0.35\textwidth}
		\includegraphics[width=1\linewidth]{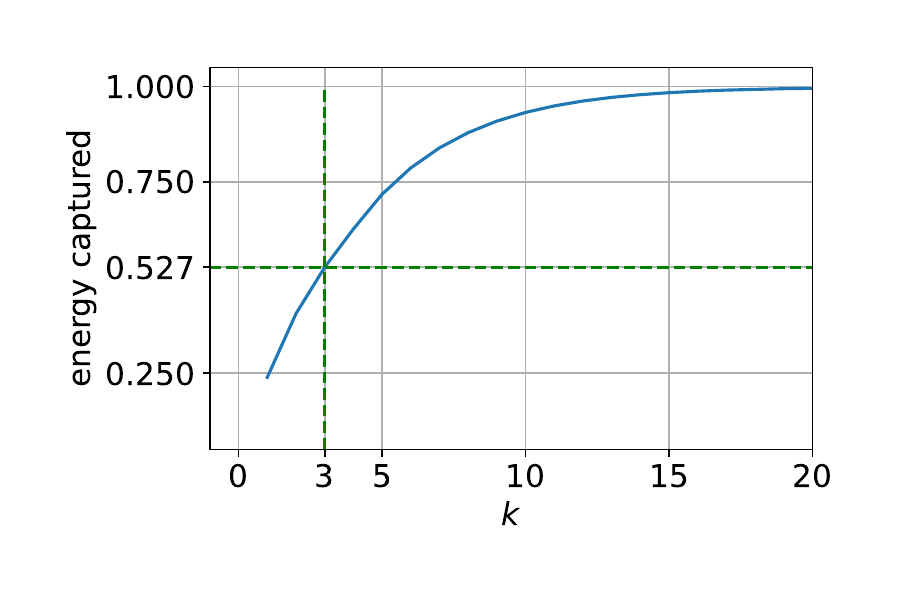}
	\end{subfigure}
	\caption{Wave example: The figure shows the delay of the singular values obtained using the training data. The right plot shows the energy captured using $3$ dominant modes.}
	\label{fig:wave_svd}
\end{figure}

Next, using $\hat\bq$ and $\hat\bp$, we learn the desired embeddings using \linearembs, \quadembs, and \cubicembs\ to learn the dynamics of the POD coordinates through these embeddings. For comparison, we also include the \opinf\ approach, which we discussed in the previous example.
To evaluate the performance of these approaches, we utilize three left-out test initial conditions. Note that the test initial conditions and the dynamics evolution using the test initial conditions are high-dimensional, which are then projected using the same projection matrix determined in the training phase to obtain the POD coordinates.
In \Cref{fig:wave_test_err}, we present a qualitative evaluation of their performances based on the error measure \eqref{eq:error_measure}.
It is interesting to note that all methods show remarkably similar performance. Surprisingly, \opinf\ with linear models is able to capture the dynamics of the POD coordinates well, indicating that there is no particular need to learn a nonlinear coordinate transformation since the POD coordinates themselves are good, and their dynamic evaluation can be accurately learned using a linear system. 

\begin{figure}[tb]
	\centering
	\includegraphics[width=0.4\linewidth]{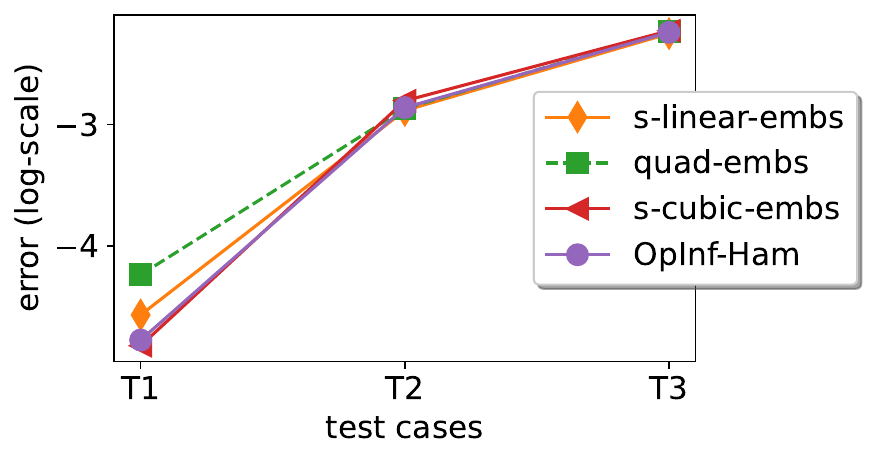}				
	\caption{Wave example: A comparison of all three considered methodologies on test cases.}
	\label{fig:wave_test_err}
\end{figure}

Next, using the POD coordinates, we focus on constructing the solution on the full-spatial domain using the POD coordinates. Similar to the previous example, we employ three methods: linear-decoder, quad-decoder, and convo-decoder. Although all four approaches, namely \linearembs, \quadembs, \cubicembs, and \opinf\ perform competitively to learn the dynamics for POD coordinates, we use the recovered POD coordinates using \linearembs, to be consistent with the previous example. For two of the test scenarios, in \Cref{fig:wave_decoder}, we present the reconstructed solution on the entire domain using various decoder approaches. In \Cref{tab:wave_err_dec}, we also report the mean $L_2$-errors between different reconstructed solutions and the ground truth.  

Our observation reveals that the linear-decoder performs very poorly, which is to be expected, considering that the data has slow-decaying Kolmogorov $N$-widths \cite{greif2019decay} and the three POD basis capture only $52.7\%$ of the energy present in the training data. With the quad-decoder, the quality of the reconstruction improves as compared to the linear-decoder but still exhibits artifacts; see the third column in \Cref{fig:wave_decoder}. On the other hand, the convo-decoder outperforms the other two considered approaches to reconstruct the solution, uncovering the power of neural networks in this context and helping in breaking Kolmogorov $N$-width barrier \cite{peherstorfer2022breaking}. 

\begin{figure}[tb]
	\centering
	\begin{subfigure}[b]{1\textwidth}
		\includegraphics[width=1\linewidth]{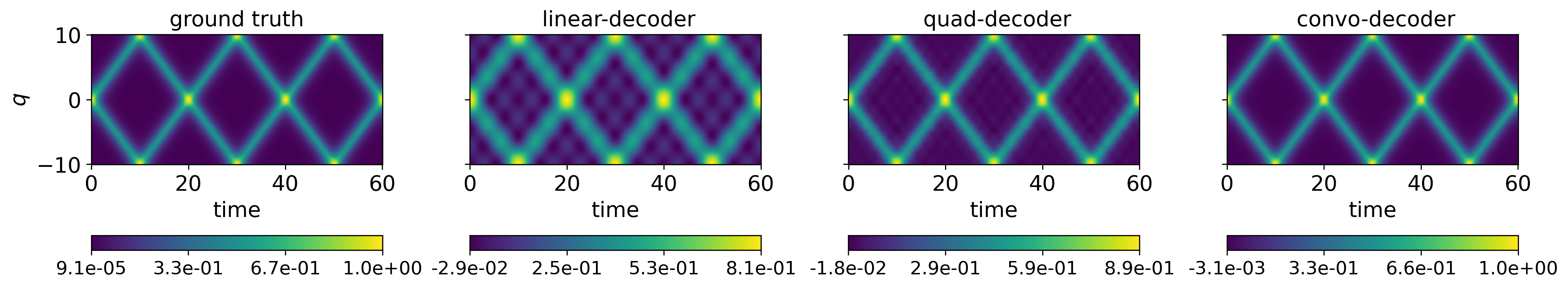}			
		\includegraphics[width=1\linewidth]{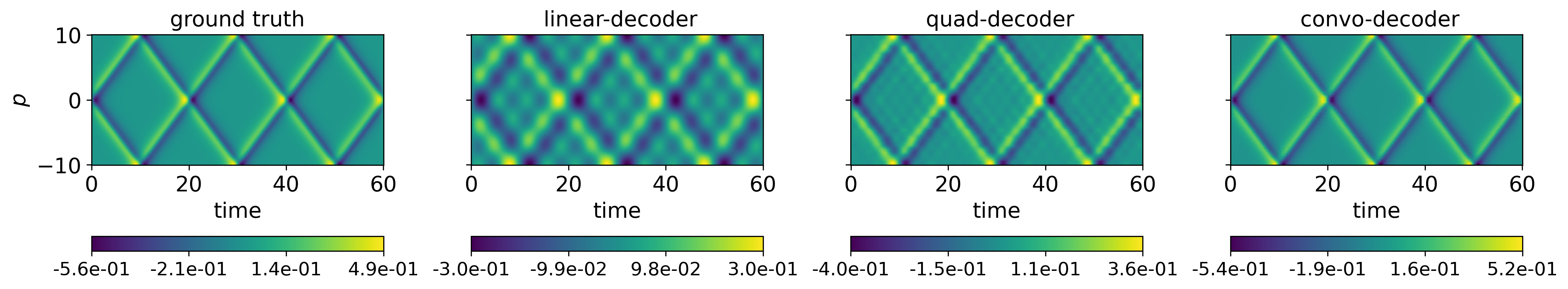}
		\caption{For testing case, $\mu = 1.0$.}
	\end{subfigure}
	\begin{subfigure}[b]{1\textwidth}
		\includegraphics[width=1\linewidth]{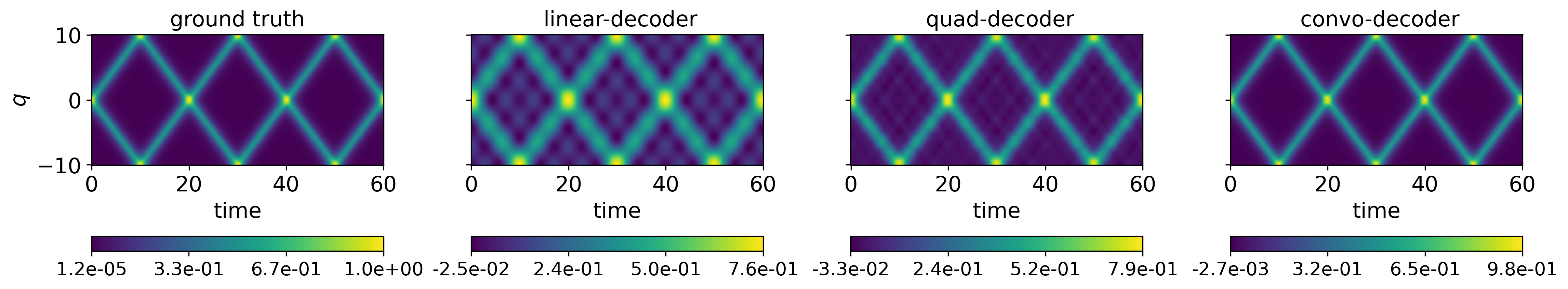}			
		\includegraphics[width=1\linewidth]{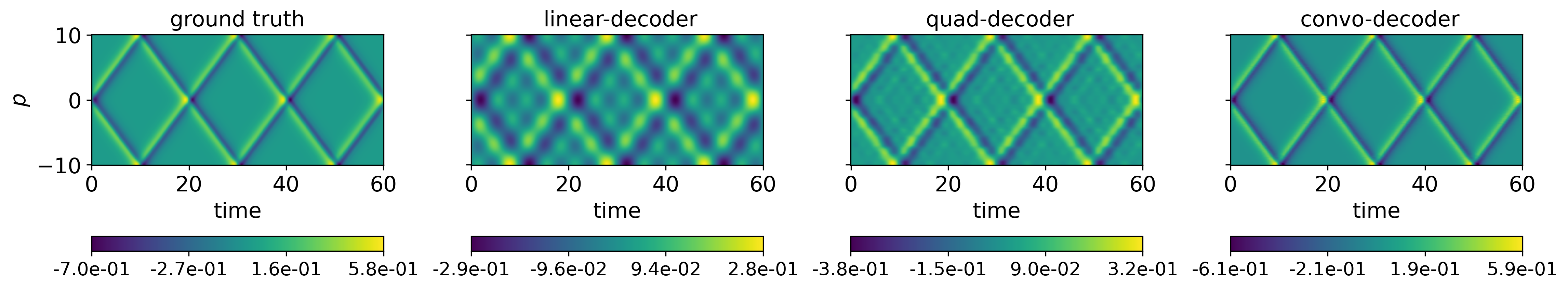}
		\caption{For testing case, $\mu = 1.2$.}
	\end{subfigure}
	\caption{Wave example: A comparison of the reconstruction of $\bq$ and $\bp$ using the POD coordinates $\hat\bq$ and $\hat\bp$.}
	\label{fig:wave_decoder}
\end{figure}

\begin{table}[tb]
	\begin{tabular}{|c|c|c|c|} \hline
		& \multicolumn{3}{c|}{Methods}                                                                                \\ \hline
		& linear-decoder & quad-decoder & convo-decoder  \\ \hline
		Test for $\mu = 0.7$ & { $2.01 \cdot 10^{-3} $}                          & $ 1.88\cdot 10^{-4}$                        & $\mathbf{1.40\cdot 10^{-5}}$                                             \\ \hline
		Test for  $\mu = 1.0$  & { $5.26\cdot 10^{-3}$ }                         & $9.54\cdot 10^{-4}   $                     & $\mathbf{1.06\cdot 10^{-4}}$                          \\   \hline     
		Test for $\mu = 1.2$  & { $9.18\cdot 10^{-3}$ }                         & $2.84\cdot 10^{-3}   $                     & $\mathbf{4.99\cdot 10^{-4}}$                          \\   \hline                     
	\end{tabular} 
	\caption{Wave model: The table shows the mean $L_2$-errors of the reconstructed and ground truth solutions using various decoder approaches.}
	\label{tab:wave_err_dec}
\end{table}

%% file: Canconical_Hamiltonian_Lifting_Stable.bbl
\begin{thebibliography}{10}

\bibitem{leimkuhler2005simulating}
B.~Leimkuhler and S.~Reich, ``Simulating {Hamiltonian} mechanics,'' {\em
  Cambridge Monogr. Appl. Comput. Math.}, 2005.

\bibitem{marsden2013introduction}
J.~E. Marsden and T.~S. Ratiu, {\em Introduction to mechanics and symmetry: {A}
  basic exposition of classical mechanical systems}, vol.~17.
\newblock Springer Science \& Business Media, 2013.

\bibitem{arnol2013mathematical}
V.~I. Arnol'd, {\em Mathematical Methods of Classical Mechanics}, vol.~60.
\newblock Springer Science \& Business Media, 2013.

\bibitem{NarP90}
S.~N. Kumpati and P.~Kannan, ``Identification and control of dynamical systems
  using neural networks,'' {\em IEEE Trans. Neural Netw.}, vol.~1, no.~1,
  pp.~4--27, 1990.

\bibitem{rico1994continuous}
R.~Rico-Martinez, J.~Anderson, and I.~Kevrekidis, ``Continuous-time nonlinear
  signal processing: a neural network based approach for gray box
  identification,'' in {\em Proc. IEEE Workshop on Neural Netw. for Signal
  Processing}, pp.~596--605, IEEE, 1994.

\bibitem{VanM96}
P.~{Van~Overschee} and B.~{de Moor}, {\em Subspace Identification of Linear
  Systems: Theory, Implementation, Applications}.
\newblock Kluwer Academic Publishers, 1996.

\bibitem{SuyVdM96}
J.~A. Suykens, J.~P. Vandewalle, and B.~L. de~Moor, {\em Artificial Neural
  Networks for Modelling and Control of Non-Linear Systems}.
\newblock Springer, 1996.

\bibitem{lennart1999system}
L.~Ljung, {\em System Identification: Theory for the User}.
\newblock Prentice Hall, NJ, 1999.

\bibitem{brunton2016discovering}
S.~L. Brunton, J.~L. Proctor, and J.~N. Kutz, ``Discovering governing equations
  from data by sparse identification of nonlinear dynamical systems,'' {\em
  Proc. Nat. Acad. Sci. U.S.A.}, vol.~113, no.~15, pp.~3932--3937, 2016.

\bibitem{peherstorfer2016data}
B.~Peherstorfer and K.~Willcox, ``Data-driven operator inference for
  nonintrusive projection-based model reduction,'' {\em Comp. Meth. Appl. Mech.
  Eng.}, vol.~306, pp.~196--215, 2016.

\bibitem{chen2018neural}
R.~T. Chen, Y.~Rubanova, J.~Bettencourt, and D.~K. Duvenaud, ``Neural ordinary
  differential equations,'' in {\em Adv. Neural Inform. Process. Sys.},
  pp.~6571--6583, 2018.

\bibitem{raissi2019physics}
M.~Raissi, P.~Perdikaris, and G.~E. Karniadakis, ``Physics-informed neural
  networks: A deep learning framework for solving forward and inverse problems
  involving nonlinear partial differential equations,'' {\em J. Comput. Phys.},
  vol.~378, pp.~686--707, 2019.

\bibitem{greydanus2019hamiltonian}
S.~Greydanus, M.~Dzamba, and J.~Yosinski, ``Hamiltonian neural networks,'' {\em
  Advances in neural information processing systems}, vol.~32, 2019.

\bibitem{zhong2020dissipative}
Y.~D. Zhong, B.~Dey, and A.~Chakraborty, ``Dissipative {SYMODEN}: {Encoding
  Hamiltonian} dynamics with dissipation and control into deep learning,'' {\em
  arXiv preprint arXiv:2002.08860}, 2020.

\bibitem{zhong2019symplectic}
Y.~D. Zhong, B.~Dey, and A.~Chakraborty, ``Symplectic {ODE-NET}: {Learning
  Hamiltonian} dynamics with control,'' {\em arXiv preprint arXiv:1909.12077},
  2019.

\bibitem{desai2021port}
S.~A. Desai, M.~Mattheakis, D.~Sondak, P.~Protopapas, and S.~J. Roberts,
  ``Port-{Hamiltonian} neural networks for learning explicit time-dependent
  dynamical systems,'' {\em Phys. Rev. E}, vol.~104, no.~3, p.~034312, 2021.

\bibitem{duong2021hamiltonian}
T.~Duong and N.~Atanasov, ``Hamiltonian-based neural {ODE} networks on the
  {SE(3)} manifold for dynamics learning and control,'' {\em arXiv preprint
  arXiv:2106.12782}, 2021.

\bibitem{jin2022learning}
P.~Jin, Z.~Zhang, I.~G. Kevrekidis, and G.~E. Karniadakis, ``Learning {Poisson}
  systems and trajectories of autonomous systems via {Poisson} neural
  networks,'' {\em IEEE Trans. Neural Netw. Learn. Syst.}, 2022.

\bibitem{sharma2022hamiltonian}
H.~Sharma, Z.~Wang, and B.~Kramer, ``Hamiltonian operator inference:
  Physics-preserving learning of reduced-order models for canonical
  {H}amiltonian systems,'' {\em Physica D}, vol.~431, p.~133122, 2022.

\bibitem{gruber2023canonical}
A.~Gruber and I.~Tezaur, ``Canonical and noncanonical {H}amiltonian operator
  inference,'' {\em arXiv preprint arXiv:2304.06262}, 2023.

\bibitem{Sharmaetal23}
H.~Sharma, H.~Mu, P.~Buchfink, R.~Geelen, S.~Glas, and B.~Kramer, ``Symplectic
  model reduction of {H}amiltonian systems using data-driven quadratic
  manifolds,'' 2023.
\newblock arXiv:2305.15490.

\bibitem{SK_LagrangianOPINF}
H.~Sharma and B.~Kramer, ``Preserving {L}agrangian structure in data-driven
  reduced-order modeling of large-scale mechanical systems,'' 2022.
\newblock arXiv:2203.06361.

\bibitem{koopman1931hamiltonian}
B.~O. Koopman, ``Hamiltonian systems and transformation in {H}ilbert space,''
  {\em Proc. Nat. Acad. Sci. U.S.A.}, vol.~17, no.~5, pp.~315--318, 1931.

\bibitem{Koopman32}
B.~O. Koopman and J.~V. Neumann, ``Dynamical systems of continuous spectra,''
  {\em Proc. Natl. Acad. Sci. USA}, vol.~18, no.~3, pp.~255--263, 1932.

\bibitem{rowley2009spectral}
C.~W. Rowley, I.~Mezi\'c, S.~Bagheri, P.~Schlatter, and D.~Henningson,
  ``Spectral analysis of nonlinear flows,'' {\em J. Fluild Mech.}, vol.~641,
  no.~1, pp.~115--127, 2009.

\bibitem{schmid2010dynamic}
P.~J. Schmid, ``Dynamic mode decomposition of numerical and experimental
  data,'' {\em J. Fluild Mech.}, vol.~656, pp.~5--28, 2010.

\bibitem{kutz2016dynamic}
J.~N. Kutz, S.~L. Brunton, B.~W. Brunton, and J.~L. Proctor, {\em Dynamic mode
  decomposition: data-driven modeling of complex systems}.
\newblock Philadelphia, PA: {SIAM} Publications, 2016.

\bibitem{williams2015data}
M.~O. Williams, I.~G. Kevrekidis, and C.~W. Rowley, ``A data--driven
  approximation of the {K}oopman operator: Extending dynamic mode
  decomposition,'' {\em J. Nonlinear Sci.}, vol.~25, no.~6, pp.~1307--1346,
  2015.

\bibitem{noe2013variational}
F.~No{\'e} and F.~N\"uske, ``A variational approach to modeling slow processes
  in stochastic dynamical systems,'' {\em Multiscale Modeling \& Simulation},
  vol.~11, no.~2, pp.~635--655, 2013.

\bibitem{nuske2014variational}
F.~N\"uske, B.~G. Keller, G.~P{\'e}rez-Hern{\'a}ndez, A.~S. Mey, and
  F.~No{\'e}, ``Variational approach to molecular kinetics,'' {\em J. Chem.
  Theory Comput.}, vol.~10, no.~4, pp.~1739--1752, 2014.

\bibitem{nuske2016variational}
F.~N{\"u}ske, R.~Schneider, F.~Vitalini, and F.~No{\'e}, ``Variational tensor
  approach for approximating the rare-event kinetics of macromolecular
  systems,'' {\em J. Chem. Phys.}, vol.~144, no.~5, 2016.

\bibitem{brunton2016koopman}
S.~L. Brunton, B.~W. Brunton, J.~L. Proctor, and J.~N. Kutz, ``Koopman
  invariant subspaces and finite linear representations of nonlinear dynamical
  systems for control,'' {\em PloS One}, vol.~11, no.~2, p.~e0150171, 2016.

\bibitem{Williams_kernel_15}
M.~O. Williams, C.~W. Rowley, and I.~G. Kevrekidis, ``A kernel-based method for
  data-driven {K}oopman spectral analysis,'' 2015.

\bibitem{wehmeyer2018time}
C.~Wehmeyer and F.~No{\'e}, ``Time-lagged autoencoders: Deep learning of slow
  collective variables for molecular kinetics,'' {\em J. Chem. Phys.},
  vol.~148, no.~24, 2018.

\bibitem{mardt2018vampnets}
A.~Mardt, L.~Pasquali, H.~Wu, and F.~No{\'e}, ``{VAMPnets} for deep learning of
  molecular kinetics,'' {\em Nat. Commun}, vol.~9, no.~1, p.~5, 2018.

\bibitem{Takeishietal_17}
N.~Takeishi, Y.~Kawahara, and T.~Yairi, ``Learning {K}oopman invariant
  subspaces for dynamic mode decomposition,'' in {\em Adv. Neural Inform.
  Process. Sys.} (I.~Guyon, U.~V. Luxburg, S.~Bengio, H.~Wallach, R.~Fergus,
  S.~Vishwanathan, and R.~Garnett, eds.), vol.~30, pp.~1130--1140, 2017.

\bibitem{yeung2019learning}
E.~Yeung, S.~Kundu, and N.~Hodas, ``Learning deep neural network
  representations for {K}oopman operators of nonlinear dynamical systems,'' in
  {\em American Control Conference (ACC)}, pp.~4832--4839, IEEE, 2019.

\bibitem{otto2019linearly}
S.~E. Otto and C.~W. Rowley, ``Linearly recurrent autoencoder networks for
  learning dynamics,'' {\em {SIAM} J. Appl. Dyn. Syst.}, vol.~18, no.~1,
  pp.~558--593, 2019.

\bibitem{li2017extended}
Q.~Li, F.~Dietrich, E.~M. Bollt, and I.~G. Kevrekidis, ``Extended dynamic mode
  decomposition with dictionary learning: A data-driven adaptive spectral
  decomposition of the {K}oopman operator,'' {\em Chaos}, vol.~27, no.~10,
  2017.

\bibitem{lusch2018deep}
B.~Lusch, J.~N. Kutz, and S.~L. Brunton, ``Deep learning for universal linear
  embeddings of nonlinear dynamics,'' {\em Nat. Commun.}, vol.~9, no.~1,
  p.~4950, 2018.

\bibitem{zhang2022hamiltonian}
J.~Zhang, Q.~Zhu, and W.~Lin, ``Hamiltonian neural {K}oopman operator,'' in
  {\em Workshop paper at the Symbiosis of Deep Learning and Differential
  Equations II---NeurIPS}, 2022.
\newblock Available at \url{https://openreview.net/forum?id=oeIr0pv-sMw}.

\bibitem{buchfink2021symplectic}
P.~Buchfink, S.~Glas, and B.~Haasdonk, ``Symplectic model reduction of
  {H}amiltonian systems on nonlinear manifolds and approximation with weakly
  symplectic autoencoder,'' {\em {SIAM} J. Sci. Comput.}, vol.~45, no.~2,
  pp.~A289--A311, 2023.

\bibitem{yildizetal23}
S.~Yildiz, P.~Goyal, T.~Bendokat, and P.~Benner, ``Data-driven identification
  of quadratic symplectic representations of nonlinear {H}amiltonian systems,''
  {\em arXiv preprint arXiv:2308.01084}, 2023.

\bibitem{savageau1987recasting}
M.~A. Savageau and E.~O. Voit, ``Recasting nonlinear differential equations as
  {S}-systems: a canonical nonlinear form,'' {\em Mathematical Biosciences},
  vol.~87, no.~1, pp.~83--115, 1987.

\bibitem{morGu11}
C.~Gu, ``{QLMOR}: A projection-based nonlinear model order reduction approach
  using quadratic-linear representation of nonlinear systems,'' {\em IEEE
  Trans. Comput. Aided Des. Integr. Circuits. Syst.}, vol.~30, no.~9,
  pp.~1307--1320, 2011.

\bibitem{qian2020lift}
E.~Qian, B.~Kramer, B.~Peherstorfer, and K.~Willcox, ``Lift \& learn:
  Physics-informed machine learning for large-scale nonlinear dynamical
  systems,'' {\em Phys. D}, vol.~406, p.~132401, 2020.

\bibitem{goyal2022generalized}
P.~Goyal and P.~Benner, ``Generalized quadratic-embeddings for nonlinear
  dynamics using deep learning,'' {\em arXiv preprint arXiv:2211.00357}, 2022.

\bibitem{hairer2006structure}
E.~Hairer, C.~Lubich, and G.~Wanner, {\em Geometric numerical integration},
  vol.~31 of {\em Springer Series in Computational Mathematics}.
\newblock Springer-Verlag, Berlin, 2002.

\bibitem{peng2016symplectic}
L.~Peng and K.~Mohseni, ``Symplectic model reduction of {H}amiltonian
  systems,'' {\em {SIAM} J. Sci. Comput.}, vol.~38, no.~1, pp.~A1--A27, 2016.

\bibitem{karasozen2018energy}
B.~Karas{\"o}zen and M.~Uzunca, ``Energy preserving model order reduction of
  the nonlinear {S}chr{\"o}dinger equation,'' {\em Advances in Computational
  Mathematics}, vol.~44, no.~6, pp.~1769--1796, 2018.

\bibitem{geelen2023operator}
R.~Geelen, S.~Wright, and K.~Willcox, ``Operator inference for non-intrusive
  model reduction with quadratic manifolds,'' {\em Comp. Meth. Appl. Mech.
  Eng.}, vol.~403, p.~115717, 2023.

\bibitem{bridges2006}
T.~J. Bridges and S.~Reich, ``Numerical methods for {Hamiltonian PDEs},'' {\em
  Journal of Physics A: mathematical and general}, vol.~39, no.~19, p.~5287,
  2006.

\bibitem{kolmogoroff1936uber}
A.~Kolmogoroff, ``Uber die beste annaherung von funktionen einer gegebenen
  {F}unktionenklasse,'' {\em Annals of Mathematics}, pp.~107--110, 1936.

\bibitem{greif2019decay}
C.~Greif and K.~Urban, ``Decay of the {K}olmogorov {N}-width for wave
  problems,'' {\em Appl. Math. Lett.}, vol.~96, pp.~216--222, 2019.

\bibitem{peherstorfer2022breaking}
B.~Peherstorfer, ``Breaking the {K}olmogorov barrier with nonlinear model
  reduction,'' {\em Notices Amer. Math. Soc.}, vol.~69, no.~5, pp.~725--733,
  2022.

\bibitem{elsken2019neural}
T.~Elsken, J.~H. Metzen, and F.~Hutter, ``Neural architecture search: A
  survey,'' {\em J. Mach. Learn. Res.}, vol.~20, no.~1, pp.~1997--2017, 2019.

\bibitem{rudy2019deep}
S.~H. Rudy, J.~N. Kutz, and S.~L. Brunton, ``Deep learning of dynamics and
  signal-noise decomposition with time-stepping constraints,'' {\em J. Comput.
  Phys.}, vol.~396, pp.~483--506, 2019.

\bibitem{goyal2022neural}
P.~Goyal and P.~Benner, ``Neural ordinary differential equations with irregular
  and noisy data,'' {\em Roy. Soc. Open Sci.}, vol.~10, no.~7, p.~221475, 2023.

\bibitem{kingma2014adam}
D.~P. Kingma and J.~Ba, ``Adam: A method for stochastic optimization,'' {\em
  arXiv preprint arXiv:1412.6980}, 2014.

\bibitem{jain2017quadratic}
S.~Jain, P.~Tiso, J.~B. Rutzmoser, and D.~J. Rixen, ``A quadratic manifold for
  model order reduction of nonlinear structural dynamics,'' {\em Computers \&
  Structures}, vol.~188, pp.~80--94, 2017.

\end{thebibliography}
